\documentclass{article}

\usepackage[preprint]{nips_2018}
\usepackage[numbers]{natbib}

\usepackage{subfigure}
\usepackage[utf8]{inputenc} % allow utf-8 input
\usepackage[T1]{fontenc}    % use 8-bit T1 fonts
\usepackage{hyperref}       % hyperlinks
\usepackage{url}            % simple URL typesetting
\usepackage{booktabs}       % professional-quality tables
\usepackage{amsfonts}       % blackboard math symbols
\usepackage{nicefrac}       % compact symbols for 1/2, etc.
\usepackage{microtype}      % microtypography
\usepackage{amsmath}
\usepackage{algorithm,algorithmic}
\usepackage{graphicx}
\usepackage{amsmath,amssymb,amsthm}
\usepackage{algorithm,algorithmic}

\usepackage{xcolor}
\usepackage{framed}
\usepackage{graphicx}
\graphicspath{ {images/} }
\usepackage{tikz}
\usepackage{pdfpages}

\usepackage[T1]{fontenc}
\usepackage[utf8]{inputenc}
\usepackage{tabularx,ragged2e,booktabs,caption}
\newcolumntype{C}[1]{>{\Centering}m{#1}}

\usepackage{caption}

\usepackage{thmtools}
\usepackage{thm-restate}
\usepackage{cleveref}

\title{Implicit Policy for Reinforcement Learning}

\author{
   Yunhao Tang \\
  Columbia University \\
  yt2541@columbia.edu \\
  \And
  Shipra Agrawal \footnote{} \\
  Columbia University \\
  sa3305@columbia.edu \\
}

\begin{document}
% \nipsfinalcopy is no longer used

\maketitle

\begin{abstract}
We introduce \emph{Implicit Policy}, a general class of expressive policies that can flexibly represent complex action distributions in reinforcement learning, with efficient algorithms to compute entropy regularized  policy gradients. We empirically show that, despite its simplicity in implementation, entropy regularization combined with a rich policy class can attain desirable properties displayed under maximum entropy reinforcement learning framework, such as robustness and multi-modality. 
\end{abstract}

\section{Introduction}

Reinforcement Learning (RL) combined with deep neural networks have led to a wide range of successful applications, including the game of Go, robotics control and video game playing \cite{silver2016,schulman2015,mnih2013}. During the training of deep RL agent, the injection of noise into the learning procedure can usually prevent the agent from premature convergence to bad locally optimal solutions, for example, by entropy regularization \cite{schulman2015,mnih2016} or by explicitly optimizing a maximum entropy objective \cite{tuomas2017,nachum2017}.

Though entropy regularization is much simpler to implement in practice, it greedily optimizes the policy entropy at each time step, without accounting for future effects. On the other hand, maximum entropy objective considers the entropy of the distribution over entire trajectories, and is more conducive to theoretical analysis \cite{asadi2017}. Recently, \cite{tuomas2017,tuomas2018} also shows that optimizing the maximum entropy objective can lead to desirable properties such as robustness and multi-modal policy.

Can we preserve the simplicity of entropy regularization while attaining desirable properties under maximum entropy framework? To achieve this, a necessary condition is an expressive representation of policy. Though various flexible probabilistic models have been proposed in generative modeling \cite{goodfellow2015,dustin2017}, such models are under-explored in policy based RL. To address such issues, we propose flexible policy classes and efficient algorithms to compute entropy regularized policy gradients.

In Section 3, we introduce \emph{Implicit Policy}, a generic policy representation from which we derive two expressive policy classes, Normalizing Flows Policy (\emph{NFP}) and more generally, Non-invertible Blackbox Policy (\emph{NBP}). \emph{NFP} provides a novel architecture that embeds state information into Normalizing Flows; \emph{NBP} assumes little about policy architecture, yet we propose algorithms to efficiently compute entropy regularized policy gradients when the policy density is not accessible. In Section 4, we show that entropy regularization optimizes a lower bound of maximum entropy objective. In Section 5, we show that when combined with entropy regularization, expressive policies achieve competitive performance on benchmarks and leads to robust and multi-modal policies.

\section{Preliminaries}
\subsection{Background}
We consider the standard RL formalism consisting of an agent interacting with the environment. At time step $t\geq0$, the agent is in state $s_t\in\mathcal{S}$, takes action $a_t\in\mathcal{A}$, receives instant reward $r_t\in\mathbb{R}$ and transitions to next state $s_{t+1}\sim p(s_{t+1}|s_t,a_t)$. Let $\pi:\mathcal{S}\mapsto \mathcal{A}$ be a policy. The objective of RL is to search for a policy which maximizes cumulative expected reward  $J(\pi) = \mathbb{E}_\pi\big[\sum_{t=0}^\infty r_t \gamma^t  \big]$, where $\gamma \in (0,1]$ is a discount factor. The action value function of policy $\pi$ is defined as $Q^\pi(s,a) = \mathbb{E}_\pi\big[\sum_{t=0}^\infty r_t\gamma^t | s_0 = s, a_0 = a\big]$. In policy based RL, a policy is explicitly parameterized as $\pi_\theta$ with parameter $\theta$, and the policy can be updated by policy gradients $\theta \leftarrow \theta + \alpha \nabla_\theta J(\pi_\theta)$, where $\alpha$ is the learning rate. So far, there are in general two ways to compute policy gradients for either on-policy or off-policy updates.
\paragraph{Score function gradient \& Pathwise gradient.} Given a stochastic policy $a_t \sim \pi_\theta(\cdot|s_t)$, the score function gradient for on-policy update is computed as $\nabla_\theta J(\pi_\theta) = \mathbb{E}_{\pi_\theta} \big[\sum_{t=0}^\infty Q^{\pi_\theta}(s_t,a_t) \nabla_\theta \log \pi_\theta(a_t|s_t)\big]$ as in \cite{schulman2017,schulman2015,mnih2016,sutton1999}. For off-policy update, it is necessary to introduce importance sampling weights to adjust the distribution difference between the behavior policy and current policy. Given a deterministic policy $a_t = \pi_\theta(s_t)$, the pathwise gradient for on-policy update is computed as $\nabla_\theta J(\pi_\theta) = \mathbb{E}_{\pi_\theta}\big[\sum_{t=0}^\infty \nabla_a Q^{\pi_\theta}(s_t,a) |_{a=\pi_\theta(s_t)}\nabla_\theta \pi_\theta(s_t)\big]$. In practice, this gradient is often computed off-policy \cite{silver2014,silver2016}, where the exact derivation comes from a modified off-policy objective \cite{degris2012}. 
\paragraph{Entropy Regularization.} For on-policy update, it is common to apply entropy regularization \cite{williams1992,donoghue2017,mnih2016,schulman2017}. Let $\mathbb{H}[\pi(\cdot|s)]$ be the entropy of policy $\pi$ at state $s$. The entropy regularized update is
\begin{align}
\theta\leftarrow \theta + \alpha  \{ \nabla_\theta J(\pi_\theta) + \beta \mathbb{E}_{\pi_\theta} \big [ \nabla_\theta \sum_{t=0}^\infty \mathbb{H}[\pi_\theta(\cdot|s_t)]\gamma^t) \big] \},
\label{eq:entropyreg} 
\end{align}
where $\beta > 0$ is a regularization constant. By boosting policy entropy, this update can potentially prevent the policy from premature convergence to bad locally optimal solutions. In Section 3, we will introduce expressive policies that leverage both on-policy/off-policy updates, and algorithms to efficiently compute entropy regularized policy gradients.

\paragraph{Maximum Entropy RL.} In maximum entropy RL formulation, the objective is to maximize the cumulative reward and the policy entropy
$J_{\text{MaxEnt}}(\pi_\theta) = \mathbb{E}_{\pi_\theta}\big[\sum_{t=0}^\infty r_t\gamma^t + \beta \sum_{t=0}^\infty \mathbb{H}[\pi(\cdot|s_t)]\gamma^t\big]$, where $\beta>0$ is a tradeoff constant. Note that $\nabla_\theta J_{\text{MaxEnt}}(\pi_\theta)$ differs from the update in (\ref{eq:entropyreg}) by an exchange of expectation and gradient. The intuition of $J_{\text{MaxEnt}}(\pi_\theta)$ is to achieve high reward while being as random as possible over trajectories. Since there is no simple low variance gradient estimate for $J_{\text{MaxEnt}}(\pi_\theta)$, several previous works \cite{schulman2017,tuomas2017,nachum2017} have proposed to optimize $J_{\text{MaxEnt}}(\pi_\theta)$ primarily using off-policy value based algorithms.

%\cite{tuomas2017} shows that policies trained using maximum entropy objective have desirable properties, such as robustness and multi-modality for exploration and downstream fine-tuning. 

\subsection{Related Work}
A large number of prior works have implemented policy gradient algorithms with entropy regularization \cite{schulman2015,schulman2017,mnih2016,donoghue2017}, which boost exploration by greedily maximizing policy entropy at each time step. In contrast to such greedy procedure, maximum entropy objective considers entropy over the entire policy trajectories \cite{tuomas2017,nachum2017,schulman2017chen}. Though entropy regularization is simpler to implement in practice, \cite{haarnoja2018composable,tuomas2017} argues in favor of maximum entropy objective by showing that trained policies can be robust to noise, which is desirable for real life robotics tasks; and multi-modal, a potentially desired property for exploration and fine-tuning for downstream tasks. However, their training procedure is fairly complex, which consists of training a soft Q function by fixed point iteration and a neural sampler by Stein variational gradient \cite{liu2016}. We argue that properties as robustness and multi-modality are attainable through simple entropy regularized policy gradient algorithms combined with expressive policy representations. 

Prior works have studied the property of maximum entropy objective \cite{nachum2017,ziebart2010}, entropy regularization \cite{donoghue2017} and their connections with variants of operators \cite{asadi2017}. It is commonly believed that entropy regularization greedily maximizes local policy entropy and does not account for how a policy update impacts future states. In Section 4, we show that entropy regularized policy gradient update maximizes a lower bound of maximum entropy objective, given constraints on the differences between consecutive policy iterates. This partially justifies why simple entropy regularization combined with expressive policy classes can achieve competitive empirical performance in practice. 

There is a number of prior works that discuss different policy architectures. The most common policy for continuous control is unimodal Gaussian \cite{schulman2015,schulman2017,mnih2016}. \cite{tuomas2018} discusses mixtures of Gaussian, which can represent multi-modal policies but it is necessary to specify the number of modes in advance. \cite{tuomas2017} also represents a policy using implicit model, but the policy is trained to sample from the soft Q function instead of being trained directly. Recently, we find \cite{haarnoja2018latent} also uses Normalizing Flows to represent policies, but their focus is learning an hierarchy and involves layers of pre-training. Contrary to early works, we propose to represent flexible policies using implicit models/Normalizing Flows and efficient algorithms to train the policy end-to-end. 

Implicit models have been extensively studied in probabilistic inference and generative modeling \cite{goodfellow2015,kingma2013,li2018,dustin2017}. Implicit models define distributions by transforming source noise via a forward pass of neural networks, which in general sacrifice tractable probability density for more expressive representation. Normalizing Flows are a special case of implicit models \cite{rezende2015,dinh2015,dinh2017}, where transformations from source noise to output are invertible and allow for maximum likelihood inference. Borrowing inspirations from prior works, we introduce implicit models into policy representation and empirically show that such rich policy class entails multi-modal behavior during training. In \cite{dustin2017}, GAN \cite{goodfellow2015} is used as an optimal density estimator for likelihood free inference. In our work, we apply similar idea to compute entropy regularization when policy density is not available. 

%Closely related to RL is imitation learning, where the agent does not directly interact with the environment but learns from experts \cite

\section{Implicit Policy for Reinforcement Learning}
We assume the action space $\mathcal{A}$ to be a compact subset of $\mathbb{R}^m$. Any sufficiently smooth stochastic policy can be represented as a blackbox $f_\theta(\cdot)$ with parameter $\theta$ that incorporates state information $s$ and independent source noise $\epsilon$ sampled from a simple distribution $\rho_0(\cdot)$. In state $s$, the action $a$ is sampled by a forward pass in the blackbox.
\begin{align}
a= f_\theta(s,\epsilon), \epsilon \sim \rho_0(\cdot).
\label{eq:genericpolicy}
\end{align}
For example, Gaussian policy is reduced to $a = \sigma_\theta(s) \cdot \epsilon + \mu_\theta(s)$ where $\rho_0$ is standard Gaussian \cite{schulman2015}. In general, the distribution of $a_t$ is implicitly defined: for any set $A$ of  $\mathcal{A}$, $\mathbb{P}(a \in A|s) = \int_{\epsilon:f_\theta(s,\epsilon) = a} \rho_0(\epsilon) d\epsilon$. Let $\pi_\theta(\cdot|s)$ be the density of this distribution\footnote{In future notations, when the context is clear, we use $\pi_\theta(\cdot|s)$ to denote both the density of the policy as well as the policy itself: for example, $a\sim \pi_\theta(\cdot|s)$ means sampling $a$ from the policy; $\log \pi_\theta(a|s)$ means the log density of policy at $a$ in state $s$.}. We call such policy \emph{Implicit Policy} as similar ideas have been previous explored in implicit generative modeling literature \cite{goodfellow2015,li2018,dustin2017}. In the following, we derive two expressive stochastic policy classes following this blackbox formulation, and propose algorithms to efficiently compute entropy regularized policy gradients.
\subsection{Normalizing Flows Policy (NFP)}
We first construct a stochastic policy with Normalizing Flows. Normalizing Flows  \cite{rezende2015,dinh2017} have been applied in variational inference and probabilistic modeling to represent complex distributions. In general, consider transforming a source noise $\epsilon \sim \rho_0(\cdot)$ by a series of invertible nonlinear function $g_{\theta_i}(\cdot),1\leq i\leq K$ each with parameter $\theta_i$, to output a target sample $x$,
\begin{align}
x = g_{\theta_K} \circ g_{\theta_{K-1}} \circ ... \circ g_{\theta_2} \circ g_{\theta_1} (\epsilon).
\label{eq:normflow}
\end{align}
Let $\Sigma_i$ be the Jacobian matrix of $g_\theta(\cdot)$, then the density of $x$ is computed by chain rule,
\begin{align}
\log p(x) = \log p(\epsilon) + \sum_{i=1}^K \log \text{det}(\Sigma_i).
\label{eq:chainrule}
\end{align}
For a general invertible transformation $g_{\theta_i}(\cdot)$, computing $\text{det}(\Sigma_i)$ is expensive. We follow the architecture of \cite{dinh2015} to ensure that $\text{det}(\Sigma_i)$ is computed in linear time. To combine state information, we embed state $s$ by another neural network $L_{\theta_s}(\cdot)$ with parameter $\theta_s$ and output a state vector $ L_{\theta_s}(s)$ with the same dimension as $\epsilon$. We can then insert the state vector between any two layers of (\ref{eq:normflow}) to make the distribution conditional on state $s$. In our implementation, we insert the state vector after the first transformation (we detail our architecture design in Appendix C).
\begin{align}
a = g_{\theta_K} \circ g_{\theta_{K-1}} \circ ... \circ g_{\theta_2} \circ (L_{\theta_s}(s) + g_{\theta_1} (\epsilon)).
\label{eq:statenormflow}
\end{align}
Though the additive form of $L_{\theta_s}(s)$ and $g_{\theta_1}(\epsilon)$ may in theory limit the capacity of the model, in practice we find the resulting policy still very expressive. For simplicity, we denote the above transformation (\ref{eq:statenormflow}) as $a = f_\theta(s,\epsilon)$ with parameter $\theta = \{\theta_s,\theta_i,1\leq i\leq K\}$. It is obvious that $\epsilon \leftrightarrow a = f_\theta(s,\epsilon)$ is still invertible between $a$ and $\epsilon$, which is critical for computing $\log \pi_\theta(a|s)$ according to (\ref{eq:chainrule}). Such representations build complex policy distributions with explicit probability density $\pi_\theta(\cdot|s)$, and hence entail training using score function gradient estimators.

Since there is no analytic form for entropy, we use samples to estimate entropy by re-parameterization,
 $\mathbb{H}\big[\pi_\theta(\cdot|s)\big] = \mathbb{E}_{a \sim \pi_\theta(\cdot|s)}\big[-\log \pi_\theta(a|s)\big] = \mathbb{E}_{\epsilon \sim \rho_0(\cdot)}\big[-\log \pi_\theta(f_\theta(s,\epsilon)|s)\big]$. The gradient of entropy can be easily computed by a pathwise gradient and easily implemented using back-propagation
 $\nabla_\theta \mathbb{H}\big[\pi_\theta(\cdot|s)\big]  = \mathbb{E}_{\epsilon \sim \rho_0(\cdot)}\big[-\nabla_\theta \log \pi_\theta(f_\theta(s,\epsilon)|s)\big]$. 
\paragraph{On-policy algorithm for NFP.}  Any on-policy policy optimization algorithms can be easily combined with NFP. Since NFP has explicit access to policy density, it allows for training using score function gradient estimators with efficient entropy regularization. 
%(\textcolor{red}{mention that NFP can be trained like NBP and NBP has an on-policy variant. Unimodal policy is normally factorized, no correlation between actions.})
\subsection{Non-invertible Blackbox Policy (NBP)}
The forward pass in (\ref{eq:genericpolicy}) transforms the simple noise distribution $\epsilon \sim \rho_0(\cdot)$ to complex action distribution $a_t  \sim \pi_\theta(\cdot|s_t)$ through the blackbox $f_\theta(\cdot)$. However, the mapping $\epsilon \mapsto a_t$ is in general non-invertible and we do not have access to the density $\pi_\theta(\cdot|s_t)$. We derive a pathwise gradient for such cases and leave all the proof in Appendix A.
\begin{restatable}[Stochastic Pathwise Gradient]{thm}{pathgrad}
\label{thm:pathgrad}
Given an implicit stochastic policy $a_t = f_\theta(s_t,\epsilon), \epsilon \sim\rho_0(\cdot)$. Let $\pi_\theta$ be the implicitly defined policy. Then the pathwise policy gradient for the stochastic policy is
\begin{align}
%\nabla_\theta J(\theta) = \mathbb{E}_{\pi_\theta}\big[ \mathbb{E}_{\epsilon \sim \rho_0(\cdot)}\big[\nabla_a Q^{\pi_\theta}(s_t,a)|_{a=f(s_t,\epsilon)} \nabla_\theta f_\theta(s_t,\epsilon) \big]\big]
\nabla_\theta J(\pi_\theta) &= \mathbb{E}_{\pi_\theta}\big[\mathbb{E}_{\epsilon \sim \rho_0(\cdot)}[\nabla_\theta f_\theta(s,\epsilon) \nabla_a Q^{\pi_\theta}(s,a)|_{a = f_\theta(s,\epsilon)}]\big].
%\nabla_\theta J = \int_\mathcal{S} \rho_\pi(s) \mathbb{E}_{\epsilon \sim \rho_0(\cdot)}\big[ \nabla_\theta f_\theta(s,\epsilon) \nabla_a Q^\pi(s,a) |_{a = f_\theta(s,\epsilon)} \big] ds
 \label{eq:pathwisepolicy}
\end{align}
\end{restatable}
To compute the gradient of policy entropy for such general implicit policy, we propose to train an additional classifier $c_\psi: \mathcal{S} \times \mathcal{A} \mapsto \mathbb{R}$ with parameter $\psi$ along with policy $\pi_\theta$. The classifier $c_\psi$ is trained to minimize the following objective given a policy $\pi_\theta$
\begin{align}
\min_\psi \ \mathbb{E}_{a \sim \pi_\theta(\cdot|s)}\big[ -\log \sigma (c_\psi(a,s))  \big] + \mathbb{E}_{a\sim U(\mathcal{A})}\big[-\log (1-\sigma(c_\psi(a,s)))\big],
 \label{eq:critic}
\end{align}
where $\mathcal{U}(\mathcal{A})$ is a uniform distribution over $\mathcal{A}$ and $\sigma(\cdot)$ is the sigmoid function. We have \cref{lem:classifier} in Appendix A.2 to guarantee that the optimal solution $\psi^\ast$ of (\ref{eq:critic}) provides an estimate of policy density, $c_{\psi^\ast}(s,a) = \log \frac{\pi_\theta(a|s)}{|\mathcal{A}|^{-1}}$. As a result, we could evaluate the entropy by simple re-parametrization $\mathbb{H}\big[\pi_\theta(\cdot|s)\big] = \mathbb{E}_{\epsilon \sim \rho_0(\cdot)}\big[-\log \pi(f_\theta(s,\epsilon)|s)\big] \approx \mathbb{E}_{\epsilon \sim \rho_0(\cdot)}\big[- c_{\psi}(f_\theta(s,\epsilon),s)\big]$. Further, we can compute gradients of the policy entropy through the density estimate as shown by the following theorem.
\begin{restatable}[Unbiased Entropy Gradient]{thm}{entgrad}
\label{thm:entgrad}
Let $\psi^\ast$ be the optimal solution from (\ref{eq:critic}), where the policy $\pi_\theta(\cdot|s)$ is given by implicit policy $a = f_\theta(s,\epsilon),\epsilon\sim \rho_0(\cdot)$. The gradient of entropy $\nabla_\theta \mathbb{H}\big[\pi_\theta(\cdot|s)\big]$ can be computed as
\begin{align}
\nabla_\theta \mathbb{H}\big[\pi_\theta(\cdot|s)\big] = - \mathbb{E}_{\epsilon\sim \rho_0(\cdot)}\big[\nabla_\theta c_{\psi^\ast}(f(\theta,\epsilon), s) \big]
 \label{eq:approxentropy}
\end{align}
\end{restatable}
It is worth noting that to compute $\nabla_\theta \mathbb{H}\big[\pi_\theta(\cdot|s)\big]$, simply plugging in $c_{\psi^\ast}(a,s)$ to replace $\log \pi_\theta(a|s)$ in the entropy definition does not work in general, since the optimal solution $\psi^\ast$ of (\ref{eq:critic}) implicitly depends on $\theta$. However, fortunately in this case the additional term vanishes. The above theorem guarantees that we could apply entropy regularization even when the policy density is not accessible.
\vspace{-.05in}
\paragraph{Off-policy algorithm for NBP.}  We develop an off-policy algorithm for NBP. The agent contains an implicit $f_\theta(s,\epsilon)$ with parameter $\theta$, a critic $Q_\phi(s,a)$ with parameter $\phi$ and a classifier $c_\psi(s,a)$ with parameter $\psi$. At each time step $t$, we sample action $a_t = f_\theta(s_t,\epsilon),\epsilon\sim\rho_0(\cdot)$ and save experience tuple $\{s_t,a_t,r_t,s_{t+1}\}$ to a replay buffer $B$. During training, we sample a mini-batch of tuples from $B$, update critic $Q_\phi(s,a)$ using TD learning, update policy $f_\theta(s,\epsilon)$ using pathwise gradient (\ref{eq:pathwisepolicy}) and update classifier $c_\psi(s,a)$ by gradient descent on (\ref{eq:critic}). We also maintain target networks $f_{\theta^-}(s,\epsilon),Q_{\phi^-}(s,a)$ with parameter $\theta^-,\phi^-$ to stabilize learning \cite{mnih2013,silver2016}. The pseudocode is listed in Appendix D.

\section{Entropy Regularization and Maximum Entropy RL}
Though policy gradient algorithms with entropy regularization are easy to implement in practice, they are harder to analyze due to the lack of a global objective. Now we show that entropy regularization maximizes a lower bound of maximum entropy objective when consecutive policy iterates are close.

At each iteration of entropy regularized policy gradient algorithm, the policy parameter is updated as in (\ref{eq:entropyreg}). Following similar ideas in \cite{kakade2002approximately,schulman2015}, we now interpret such update as maximizing a linearized surrogate objective in the neighborhood of the previous policy iterate $\pi_{\theta_{\text{old}}}$. The surrogate objective is
\begin{align}
J_{\text{surr}}(\pi_\theta) = J(\pi_\theta) + \beta \mathbb{E}_{\pi_{\theta_{\text{old}}}}\big[\sum_{t=0}^\infty \mathbb{H}[\pi_{\theta}(\cdot|s_t)]\gamma^t\big].
\label{eq:surrpolicyobj}
\end{align}
The first-order Taylor expansion of (\ref{eq:surrpolicyobj}) centering at $\theta_{\text{old}}$ gives a linearized surrogate objective $J_{\text{surr}}(\pi_\theta)\approx J_{\text{surr}}(\pi_{\theta_{\text{old}}}) + \nabla_\theta J_{\text{surr}}(\pi_\theta) |_{\theta = \theta_{\text{old}}} (\theta - \theta_{\text{old}})$. Let $\delta \theta = \theta - \theta_{\text{old}}$, the entropy regularized update (\ref{eq:entropyreg}) is equivalent to solving the following optimization problem then update according to $\theta \leftarrow \theta_{\text{old}} + \delta \theta$,
\begin{align}
\min_{\delta \theta}& \big[ \nabla_\theta J_{\text{surr}}(\pi_\theta)|_{\theta = \theta_{\theta_{\text{old}}}}\big]^T \delta \theta \nonumber \\
\ \text{s.t.}&\ ||\delta \theta||_2 \leq C(\alpha,\theta_{\text{old}}), \nonumber
\end{align}
where $C(\alpha,\theta_{\text{old}})$ is a positive constant depending on both the learning rate $\alpha$ and the previous iterate $\theta_{\text{old}}$, and can be recovered from (\ref{eq:entropyreg}). The next theorem shows that by constraining the KL divergence of consecutive policy iterates, the surrogate objective (\ref{eq:surrpolicyobj}) forms a non-trivial lower bound of maximum entropy objective,
\begin{restatable}[Lower Bound]{thm}{lowerbound}
\label{thm:lowerbound}
 If $\mathbb{KL}[\pi_\theta||\pi_{\theta_{\text{old}}}] \leq \alpha$, then  
\begin{align}J_{\text{MaxEnt}}(\pi) \geq J_{\text{surr}}(\pi) - \frac{\beta \gamma\sqrt{\alpha}\epsilon}{(1-\gamma)^2},  \ \ \  \text{where}\  \epsilon= \max_s |\mathbb{H}[\pi(\cdot|s)]|.
\label{eq:lowerbound} 
\end{align}
\end{restatable}
By optimizing $J_{\text{surr}}(\pi_\theta)$ at each iteration, entropy regularized policy gradient algorithms maximize a lower bound of $J_{\text{MaxEnt}}(\pi_\theta)$. This implies that though entropy regularization is a greedier procedure than optimizing maximum entropy objective, it accounts for certain effects that the maximum entropy objective is designed to capture. Nevertheless, the optimal solutions of both optimization procedures are different. Previous works \cite{donoghue2017,tuomas2017} have shown that the optimal solutions of both procedures are energy based policies, with energy functions being fixed points of Boltzmann operator and Mellowmax operator respectively \cite{asadi2017}. In Appendix B, we show that Boltzmann operator interpolates between Bellman operator and Mellowmax operator, which asserts that entropy regularization is greedier than optimizing $J_{\text{MaxEnt}}(\pi_\theta)$, yet it still maintains uncertainties in the policy updates. 

Though maximum entropy objective accounts for long term effects of policy entropy updates and is more conducive to analysis \cite{asadi2017}, it is hard to implement a simple yet scalable procedure to optimize the objective \cite{tuomas2017,tuomas2018,asadi2017}. Entropy regularization, on the other hand, is simple to implement in both on-policy and off-policy setting. In experiments, we will show that entropy regularized policy gradients combined with expressive policies achieve competitive performance in multiple aspects.

\section{Experiments}
Our experiments aim to answer the following questions: (1) Will expressive policy be hard to train, does implicit policy provide competitive performance on benchmark tasks? (2) Are implicit policies robust to noises on locomotion tasks? (3) Does implicit policy $+$ entropy regularization entail multi-modal policies as displayed under maximum entropy framework \cite{tuomas2017}? 

To answer (1), we evaluate both NFP and NBP agent on benchmark continuous control tasks in MuJoCo \cite{todorov2012} and compare with baselines. To answer (2), we compare NFP with unimodal Gaussian policy on locomotion tasks with additive observational noises. To answer (3), we illustrate the multi-modal capacity of both policy representations on specially designed tasks illustrated below, and compare with baselines. In all experiments, for NFP, we implement with standard PPO for on-policy update to approximately enforce the KL constraint (\ref{eq:lowerbound}) as in \cite{schulman2017}; for NBP, we implement the off-policy algorithm developed in Section 3. In Appendix C and F, we detail hyper-parameter settings in the experiments and provide a small ablation study.

\subsection{Locomotion Tasks}
\paragraph{Benchmark tasks.} One potential disadvantage of expressive policies compared to simple policies (like unimodal Gaussian) is that they pose a more serious statistical challenge due to a larger number of parameters. To see if implicit policy suffers from such problems, we evaluate NFP and NBP on MuJoCo benchmark tasks. For each task, we train for a prescribed number of time steps, then report the results averaged over 5 random seeds. We compare the results with baseline algorithms, such as DDPG \cite{silver2016}, SQL \cite{tuomas2017}, TRPO \cite{schulman2015} and PPO \cite{schulman2017}, where baseline TPRO and PPO use unimodal Gaussian policies. As can be seen from Table 1, both NFP and NBP achieve competitive performances on benchmark tasks: they outperform DDPG, SQL and TRPO on most tasks. However, baseline PPO tends to come on top on most tasks. Interestingly on HalfCheetah, baseline PPO gets stuck on a locally optimal gait, which NFP improves upon by a large margin. 

\begin{minipage}{\linewidth}
\centering
%\captionof{table}{MuJoCo Benchmark Tasks} \label{tab:title} 
\begin{tabular}{ C{0.8in} C{.55in} *6{C{.45in}}}\toprule[1.5pt]
\bf Tasks & \bf Timesteps  & \bf DDPG & \bf SQL & \bf TRPO & \bf PPO & \bf NFP & \bf NBP \\\midrule
Hopper       &  $2.00 \cdot 10^6$  & $\approx 1100$ & $\approx 1500$ & $\approx 1250$ &$\mathbf{\approx 2130}$& $\approx 1640$ & $\mathbf{\approx 1880}$ \\
 HalfCheetah & $1.00 \cdot 10^7$ & $\approx 6500$ & $\mathbf{\approx 8000}$ & $\approx 1800$ &$\approx 1590$& $\approx 4000$ & $\mathbf{\approx 6560}$ \\ 
  Walker2d & $5.00 \cdot 10^6$ & $\approx 1600$ & $\approx 2100$ & $\approx 800$ &$\mathbf{\approx 3800}$& $\mathbf{\approx 3000}$ & $\approx 2450$\\ 
 Ant & $1.00 \cdot 10^7$ & $\approx 200$ & $\approx 2000$ & $\approx 0$ &$\mathbf{\approx 4440}$& $\mathbf{\approx2500}$ & $\approx 2070$\\ 
%  Humanoid(rllab) & $8.00 \cdot 10^6$ & $\approx 0$ & $\approx 500$ & $\approx 100$ &NA& $\mathbf{\approx 570}$ & $\mathbf{\approx 600}$ \\
\bottomrule[1.25pt]
\end {tabular}\par
\bigskip
\small{Table 1: A comparison of implicit policy optimization with baseline algorithms on MuJoCo benchmark tasks. For each task, we show the average rewards achieved after training the agent for a fixed number of time steps. The results for NFP and NBP are averaged over 5 random seeds. The results for DDPG, SQL and TRPO are approximated based on the figures in \cite{tuomas2018}, PPO is from OpenAI baseline implementation \cite{baselines}. We highlight the top two algorithms for each task in bold font. Both TRPO and PPO use unimodal Gaussian policies.}
\end{minipage}
%\vspace{-.1in}
\paragraph{Robustness to Noisy Observations.} We add independent Gaussian noise $\mathcal{N}(0,0.1^2)$ to each component of the observations to make the original tasks partially observable. Since PPO with unimodal Gaussian achieves leading performance on noise-free locomotion tasks across on-policy baselines (A2C \cite{mnih2016}, TRPO \cite{schulman2015}) as shown in \cite{schulman2017} and Appendix E.1, we compare NFP only with PPO with unimodal Gaussian on such noisy locomotion tasks. In Figure \ref{figure:noisyobservations}, we show the learning curves of both agents, where on many tasks NFP learns significantly faster than unimodal Gaussian. Why complex policies may add to robustness? We propose that since these control tasks are known to be solved by multiple separate modes of policy \cite{mania2018simple}, observational noises potentially blur these modes and make it harder for a unimodal Gaussian policy to learn any single mode (e.g. unimodal Gaussian puts probability mass between two neighboring modes \cite{levine2018}). On the contrary, NFP can still navigate a more complex reward landscape thanks to a potentially multi-modal policy distribution and learn effectively. We leave a more detailed study of robustness, multi-modality and complex reward landscape as interesting future work.
\vspace{-.1in}
\begin{figure}[h]
\centering
\subfigure[Hopper]{\includegraphics[width=.23\linewidth]{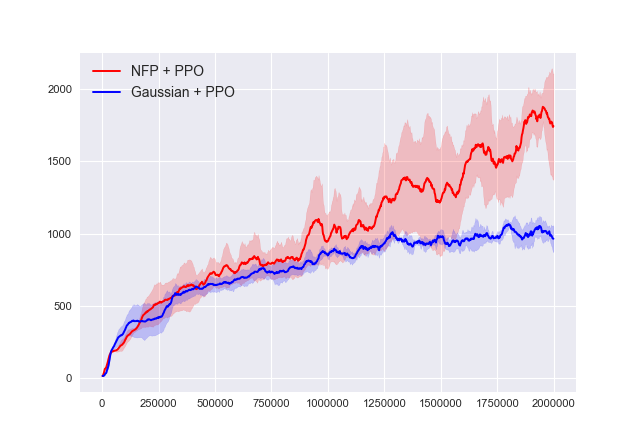}}
\subfigure[Walker]{\includegraphics[width=.23\linewidth]{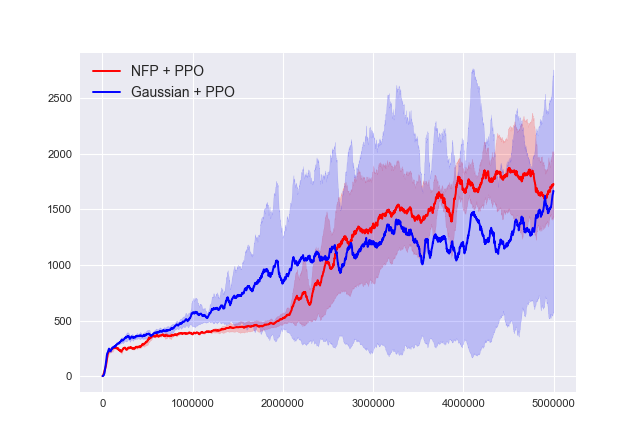}}
\subfigure[Reacher]{\includegraphics[width=.23\linewidth]{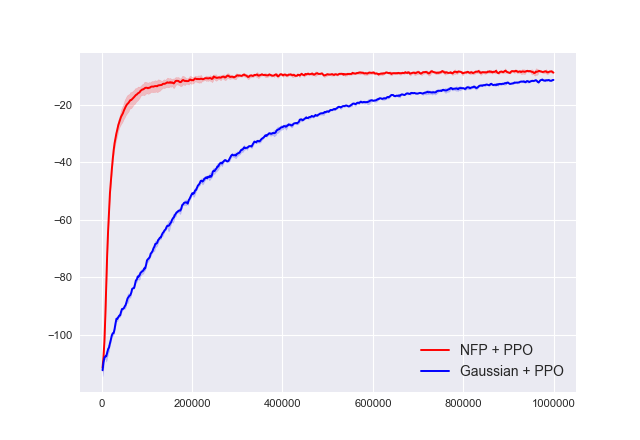}}
\subfigure[Swimmer]{\includegraphics[width=.23\linewidth]{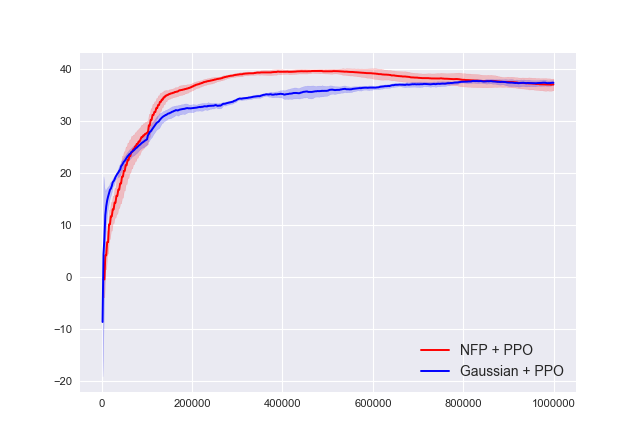}}
\subfigure[HalfCheetah]{\includegraphics[width=.23\linewidth]{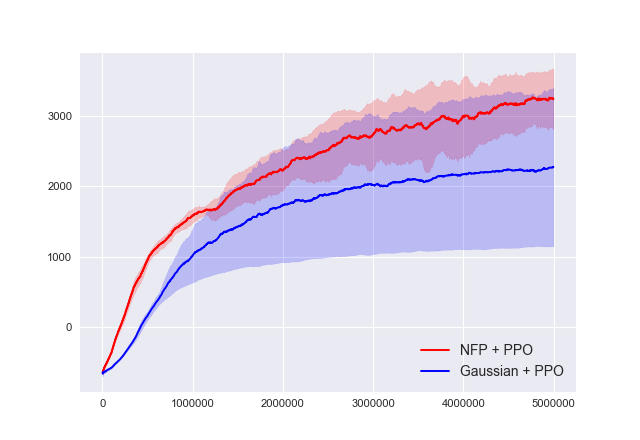}}
\subfigure[Ant]{\includegraphics[width=.23\linewidth]{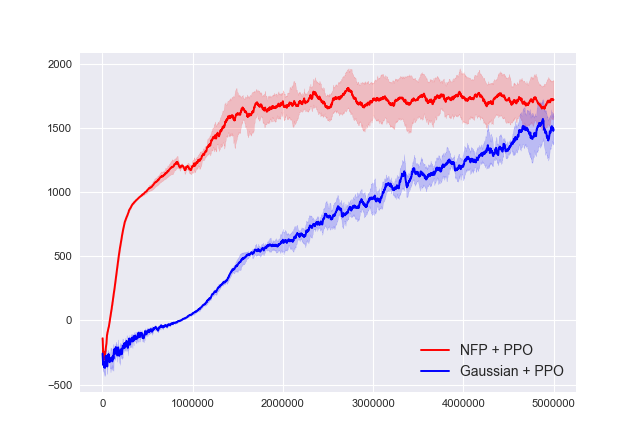}}
\subfigure[MountainCar]{\includegraphics[width=.23\linewidth]{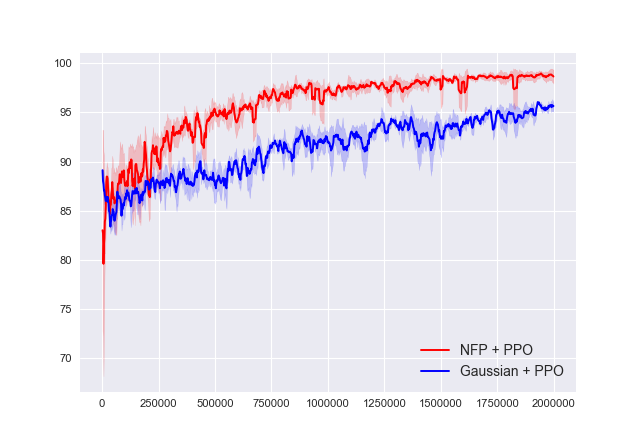}}
\subfigure[Pendulum]{\includegraphics[width=.23\linewidth]{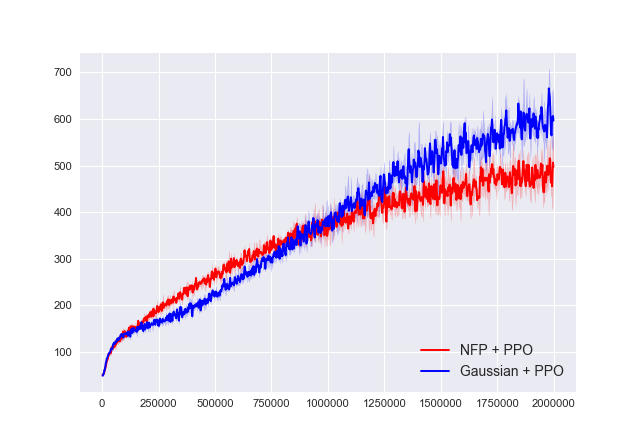}}
\caption{\small{Noisy Observations: learning curves on noisy locomotion tasks. For each task, the observation is added a Gaussian noise $\mathcal{N}(0,0.1^2)$ component-wise. Each curve is averaged over 4 random seeds. Red is NFP and blue is unimodal Gaussian, both implemented with PPO. NFP beats Gaussian on most tasks.}}
\label{figure:noisyobservations}
\end{figure}

%In Figure \ref{figure:standardize} we present how the algorithms' performance varies as the additive noise scale changes. Consider adding a Gaussian noise $\mathcal{N}(0,\sigma^2)$ for $0.0 \leq \sigma \leq 0.1$, as $\sigma$ increases the observations become less informative. We observe that NFP $+$ PPO is more robust than unimodal Gaussian $+$ PPO on many tasks. More interestingly, occasionally the performance improves as more noise is added (see Hopper). One potential reason is that adding noise has the side effect of injecting noise into the learning process, making it possible for the agent to explore more efficiently. 
%\begin{figure}[h]
%\centering
%\subfigure[Hopper]{\includegraphics[width=.23\linewidth]{hopper-noiselevel}}
%\subfigure[HalfCheetah]{\includegraphics[width=.23\linewidth]{halfcheetah-noiselevel}}
%\subfigure[Reacher]{\includegraphics[width=.23\linewidth]{reacher-noiselevel}}
%\subfigure[Swimmer]{\includegraphics[width=.23\linewidth]{swimmer-noiselevel}}
%\subfigure[MountainCar]{\includegraphics[width=.23\linewidth]{mountaincar-noiselevel}}
%\caption{\small{Relative performance as noise scale varies: all performances are standardized so that the highest achieved performance is $1$ and random performance is $0$. For each task, the observation is added a Gaussian noise $\mathcal{N}(0,\sigma^2)$ where $\sigma$ is the noise scale. Each curve is averaged over 3 random seeds. Red is NFP and blue is unimodal Gaussian, both with PPO.}}
%\label{figure:standardize}
%\end{figure}

\subsection{Multi-modal policy}
\paragraph{Gaussian Bandits.} Though factorized unimodal policies suffice for most benchmark tasks, below we motivate the importance of a flexible policy by a simple example: Gaussian bandits. Consider a two dimensional bandit $\mathcal{A} = [-1,1]^2$. The reward of action $a$ is $-a^T \Sigma^- a$ for a positive definite matrix $\Sigma$. The optimal policy for maximum entropy objective is $\pi^\ast(a) \propto \exp(-\frac{a^T \Sigma^- a}{\beta})$, i.e. a Gaussian policy with covariance matrix $\Sigma$. We compare NFP with PPO with factorized Gaussian. As illustrated in Figure \ref{figure:multigoal}(a), NFP can approximate the optimal Gaussian policy pretty closely while the factorized Gaussian cannot capture the high correlation between the two action components.
\vspace{-.1in}
\paragraph{Navigating 2D Multi-goal.} We motivate the strength of implicit policy to represent multi-modal policy by Multi-goal environment \cite{tuomas2017}. The agent has 2D coordinates as states $\mathcal{S} \subset \mathbb{R}^2$ and 2D forces as actions $\mathcal{A} \subset \mathbb{R}^2$. A ball is randomly initialized near the origin and the goal is to push the ball to reach one of the four goal positions plotted as red dots in Figure \ref{figure:multigoal}(b). While a unimodal policy can only deterministically commit the agent to one of the four goals, a multi-modal policy obtained by NBP can stochastically commit the agent to multiple goals. On the right of Figure \ref{figure:multigoal}(b) we also show sampled actions and contours of Q value functions at various states: NBP learns a very flexible policy with different number of modes in different states.
\vspace{-.1in}
\paragraph{Learning a Bimodal Reacher.} For a more realistic example, consider learning a bimodal policy for reaching one of two targets (Figure \ref{figure:pretrain}(a)). The agent has the physical coordinates of the reaching arms as states  $\mathcal{S} \subset \mathbb{R}^9$ and applies torques to the joints as actions $\mathcal{A}\subset\mathbb{R}^2$. The objective is to move the reacher head to be close to one of the targets. As illustrated by trajectories in Figure \ref{figure:multigoal}(c), while a unimodal Gaussian policy can only deterministically reach one target (red curves), a NFP agent can capture both modes by stochastically reaching one of the two targets (blue curves).
\vspace{-.1in}
\begin{figure}[h]
\centering
\subfigure[Gaussian Bandit]{\includegraphics[width=.30\linewidth]{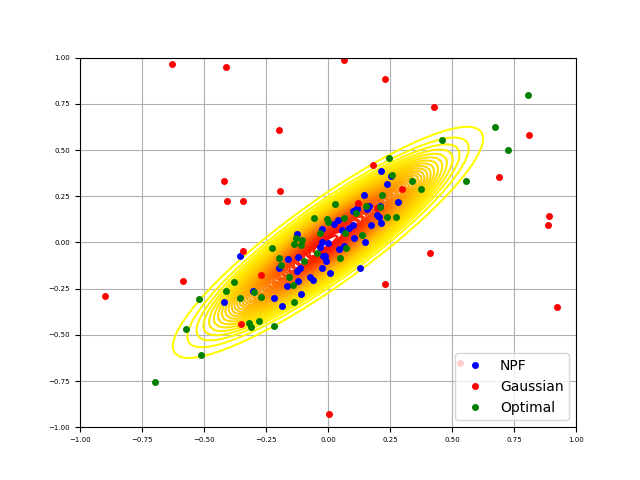}}
\subfigure[2D Multi-goal]{\includegraphics[width=.3375\linewidth]{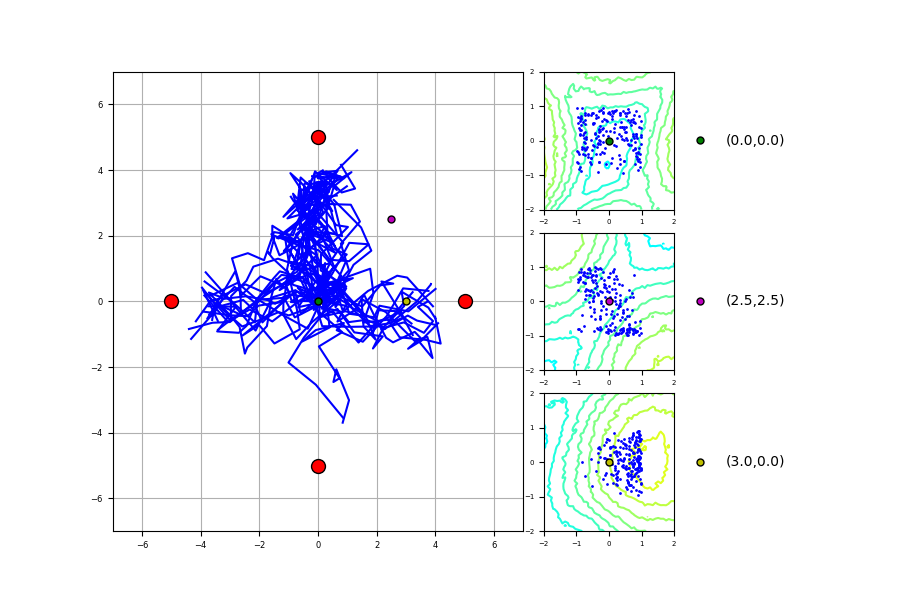}}
\subfigure[Bimodal Reacher]{\includegraphics[width=.29\linewidth]{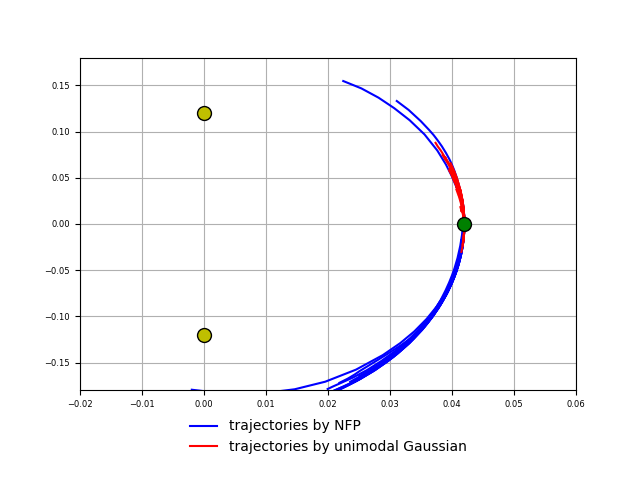}}
\caption{\small{(a): Illustration of Gaussian bandits. The $x$ and $y$ axes are actions. Green dots are actions from the optimal policy, a Gaussian distribution with covariance structure illustrated by the contours. Red dots and blue dots are actions sampled from a learned factorized Gaussian and NFP. NFP captures the covariance of the optimal policy while factorized Gaussian cannot. (b): Illustration of 2D multi-goal environment. Left: trajectories generated by trained NBP agent (solid blue curves). The $x$ and $y$ axes are coordinates of the agent (state). The agent is initialized randomly near the origin. The goals are red dots, and instant rewards are proportional to the agent's minimum distance to one of the four goals. Right: predicted Q value contours by the critic (light blue: low value, light green: high value and actions sampled from the policy (blue dots) at three selected states. The NFP policy has different number of modes at different states. (c): Trajectories of the reacher head by NFP (blue curves) and unimodal Gaussian policies (red curves) for the bimodal reacher. Yellow dots are locations of the two targets, and the green dot is the starting location of the reacher.}}
\label{figure:multigoal}
\end{figure}
\vspace{-.1in}
\paragraph{Fine-tuning for downstream tasks.} A recent paradigm for RL is to pre-train an agent to perform a conceptually high-level task, which may accelerate fine-tuning the agent to perform more specific tasks \cite{tuomas2017}. We consider pre-training a quadrupedal robot (Figure \ref{figure:pretrain}(b)) to run fast, then fine-tune the robot to run fast in a particular direction \cite{tuomas2017} as illustrated in Figure \ref{figure:pretrain}(c), where we set walls to limit the directions in which to run. Wide and Narrow Hallways tasks differ by the distance of the opposing walls. If an algorithm does not inject enough diversity during pre-training, it will commit the agent to prematurely run in a particular direction, which is bad for fine-tuning. We compare the pre-training capacity of DDPG \cite{timothy2016}, SQL \cite{tuomas2017} and NBP. As shown in Figure \ref{figure:pretrain}(d), after pre-training, NBP agent manages to run in multiple directions, while DDPG agent runs in a single direction due to a deterministic policy (Appendix E.2). In Table 2, we compare the cumulative rewards of agents after fine-tuning on downstream tasks with different pre-training as initializations. In both tasks, we find NBP to outperform DDPG, SQL and random initialization (no pre-training) by statistically significant margins, potentially because NBP agent learns a high-level running gait that is more conducive to fine-tuning. Interestingly, in Narrow Hallway, randomly initialized agent performs better than DDPG pre-training, which is probably because running fast in Narrow Hallway requires running in a very narrow direction, and DDPG pre-trained agent needs to first unlearn the overtly specialized running gait acquired from pre-training. In Wide Hallway, randomly initialized agent easily gets stuck in a locally optimal gait (running between two opposing walls) while pre-training in general helps avoid such problem.
\vspace{-.1in}
\begin{figure}[h]
\centering
\subfigure[Reacher]{\includegraphics[width=.22\linewidth]{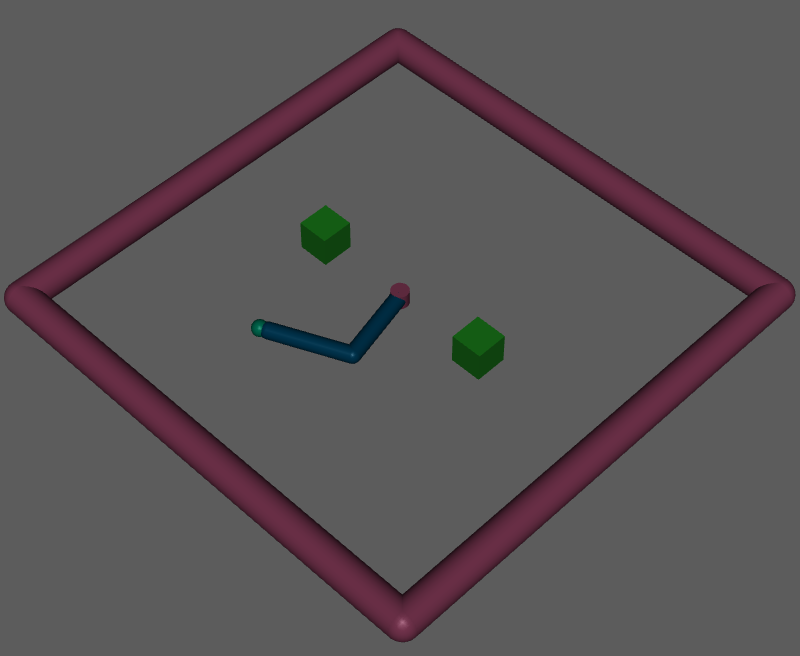}}
\subfigure[Ant]{\includegraphics[width=.24\linewidth]{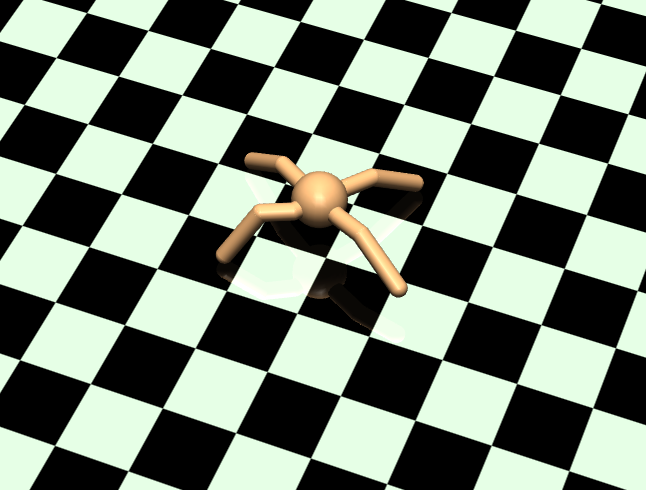}}
\subfigure[Wide Hallway]{\includegraphics[width=.24\linewidth]{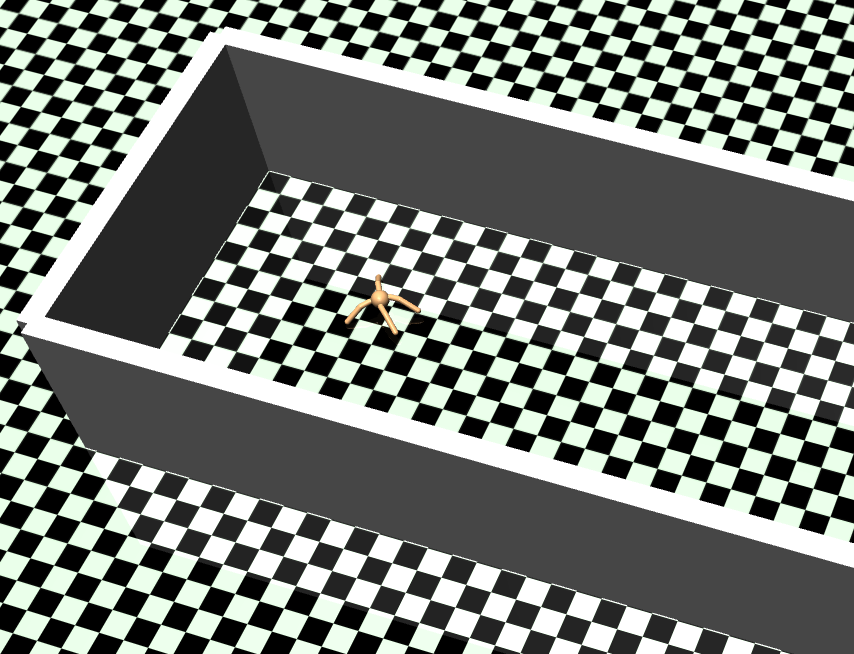}}
\subfigure[Ant Running]{\includegraphics[width=.26\linewidth]{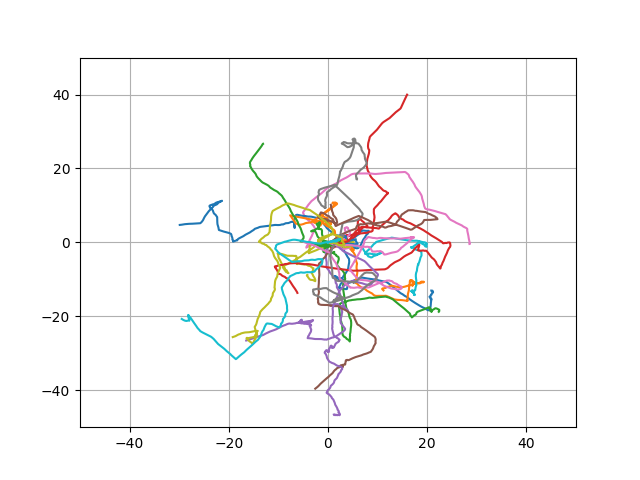}}
\caption{\small{Illustration of locomotion tasks: (a) Bimodal Reacher. Train a reacher to reach one of two targets. Green boxes are targets. (b) Ant-Running. Train a quadrupedal robot to run fast. The instant reward is the robot's center of mass velocity; (c) Ant-Hallway. Train a quadrupedal robot to run fast under the physical constraints of walls, the instant reward is the same as in (b). Narrow and Wide Hallway tasks differ by the distance between the opposing walls; (d) Trajectories by NBP agent in Ant-Running. The agent learns to run in multiple directions.}}
\label{figure:pretrain}
\end{figure}
\begin{minipage}{\linewidth}
\centering
%\captionof{table}{MuJoCo Benchmark Tasks} \label{tab:title} 
\begin{tabular}{ C{1.in} *4{C{.8in}}}\toprule[1.5pt]
\bf Tasks  & \bf Random init & \bf DDPG init & \bf SQL init &\bf NBP init \\\midrule
Wide Hallway    &  $522 \pm 111$  & $3677 \pm 472$ & $3624 \pm 312$ & $\mathbf{4306 \pm 571}$ \\
Narrow Hallway   &  $3023 \pm 280$  & $2866 \pm 248$ & $3026 \pm 352$ & $\mathbf{3752 \pm 408}$ \\
\bottomrule[1.25pt]
\end {tabular}\par
\bigskip
\centering
\small{Table 2: A comparison of downstream fine-tuning under different initializations. For each task, we show the  cumulative rewards after pre-training for $2\cdot10^6$ steps and fine-tuning for $10^6$ steps. The rewards are shown in the form $(\text{mean} \pm \text{std})$, all results are averaged over 5 seeds. Random init means the agent is trained from scratch.}
\end{minipage}
\paragraph{Combining multiple modes by Imitation Learning.} We propose another paradigm that can be of practical interest. In general, learning a multi-modal policy from scratch is hard for complex tasks since it requires good exploration and an algorithm to learn multi-modal distributions \cite{tuomas2017}, which is itself a hard inference problem \cite{goodfellow2015}. A big advantage of policy based algorithm over value based algorithm \cite{tuomas2017} is that the policy can be easily combined with imitation learning. We could decompose a complex task into several simpler tasks, each representing a simple mode of behavior easily learned by a RL agent, then combine them into a single agent using imitation learning or inverse RL \cite{abbeel2010,finn2016,ross2011}.

We illustrate with a stochastic Swimmer example (see Appendix E.3). Consider training a Swimmer to move fast either forward or backward. The aggregate behavior has two modes and it is easy to solve each single mode. We train two separate Swimmers to move forward/backward and generate expert trajectories using the trained agents. We then train a NBP / NFP agent using GAN \cite{goodfellow2015} / maximum likelihood estimation to combine both modes. Training with the same algorithms, a unimodal policy either commits to only one mode or learns a policy that puts large probability mass between the two modes \cite{levine2018,goodfellow2015}, which greatly deviates from the expert policy. On the contrary, expressive policies can more flexibly incorporate multiple modes into a single agent.

\section{Conclusion}
We have proposed \emph{Implicit Policy}, a rich class of policy that can represent complex action distributions. We have derived efficient algorithms to compute entropy regularized policy gradients for generic implicit policies. Importantly, we have also showed that entropy regularization maximizes a lower bound of maximum entropy objective, which implies that in practice entropy regularization $+$ rich policy class can lead to desired properties of maximum entropy RL. We have empirically showed that implicit policy achieves competitive performance on benchmark tasks, is more robust to observational noise, and can flexibly represent multi-modal distributions. 

%\section{Acknowledgement}
%The authors would like to thank Jalaj Bhandari for discussions on Normalizing Flows, and Sergey Levine for helpful comments on implementations of the simulation environments. The authors also thank Amazon Web Services (AWS) for computation support.

\newpage
\paragraph{Acknowledgements.} This research was supported by an Amazon Research Award (2017) and AWS cloud credits. The authors would like to thank Jalaj Bhandari for helpful discussions, and Sergey Levine for helpful comments on early stage experiments of the paper. 

\bibliographystyle{apa}
\bibliography{your_bib_file.bib}

\newpage
\appendix
\section{Proof of Theorems}
\subsection{Stochastic Pathwise Gradient}
\pathgrad*
\begin{proof}
We follow closely the derivation of deterministic policy gradient \cite{silver2014}. We assume that all conditions are satisfied to exchange expectations and gradients when necessary. Let $\pi = \pi_\theta$ denote the implicit policy $a_t = f_\theta(s_t,\epsilon), \epsilon \sim \rho_0(\cdot)$. Let $V^\pi,Q^\pi$ be the value function and action value function under such stochastic policy. We introduce $p(s\rightarrow s^\prime,k,\pi)$ as the probability of transitioning from $s$ to $s^\prime$ in $k$ steps under policy $\pi$. Overloading the notation a bit, $p(s\rightarrow s^\prime,1,a)$ is the probability of $s\rightarrow s^\prime$ in one step by taking action $a$ (i.e.,  $p(s\rightarrow s^\prime,1,a)=p(s^\prime|s,a)$). We have 
\begin{align*}
&\nabla_\theta V^\pi(s) \\
&= \nabla_\theta \mathbb{E}_{a \sim \pi(\cdot | s)}\big[ Q^\pi(s,a)\big] = \nabla_\theta \mathbb{E}_{\epsilon \sim \rho_0(\cdot)}\big[ Q^\pi(s,f_\theta(s,\epsilon))\big]\\
&= \nabla_\theta \mathbb{E}_{\epsilon \sim \rho_0(\cdot)} \big[r(s,f_\theta(s,\epsilon)) + \int_\mathcal{S} \gamma p(s^\prime |s, f_\theta(s,\epsilon)) V^\pi(s^\prime) ds^\prime \big]\\
&= \mathbb{E}_{\epsilon \sim \rho_0(\cdot)} \big[ \nabla_\theta r(s,f_\theta(s,\epsilon)) + \nabla_\theta \int_\mathcal{S} \gamma p(s^\prime |s, f_\theta(s,\epsilon)) V^\pi(s^\prime) ds^\prime \big]\\
&=  \mathbb{E}_{\epsilon \sim \rho_0(\cdot)} \big[ \nabla_\theta r(s,f_\theta(s,\epsilon)) + \int_\mathcal{S} \gamma V^\pi(s^\prime) \nabla_\theta p(s^\prime |s, f_\theta(s,\epsilon)) ds^\prime  + \int_\mathcal{S} \gamma p(s^\prime |s,f_\theta(s,\epsilon)) \nabla_\theta V^\pi(s^\prime) ds^\prime \big]\\
&= \mathbb{E}_{\epsilon \sim \rho_0(\cdot)} \big[\nabla_\theta f_\theta(s,\epsilon) \nabla_a [r(s,a) + \gamma \int_\mathcal{S} \gamma p(s^\prime| s,a)V^\pi(s^\prime)ds^\prime)]|_{a = f_\theta(s,\epsilon)} \\
&\ \ \ \ \ + \int_\mathcal{S} \gamma p(s^\prime |s,f_\theta(s,\epsilon)) \nabla_\theta V^\pi(s^\prime) ds^\prime  \big]\\
&= \mathbb{E}_{\epsilon \sim \rho_0(\cdot)} \big[ \nabla_\theta f_\theta(s,\epsilon) \nabla_a Q^\pi(s,a)|_{a=f_\theta(s,\epsilon)}\big] + \mathbb{E}_{\epsilon \sim \rho_0(\cdot)} \big[\int_{\mathcal{S}} \gamma p(s\rightarrow s^\prime, 1, f_\theta(s,\epsilon))\nabla_\theta V^\pi(s^\prime)ds^\prime \big].
\end{align*}
In the above derivation, we have used the Fubini theorem to interchange integral (expectation) and gradients. We can iterate the above derivation and have the following
\begin{align*}
&\nabla_\theta V^\pi(s)\\
&= \nabla_\theta \mathbb{E}_{a \sim \pi(\cdot | s)}\big[ Q^\pi(s,a)\big]\\
&= \mathbb{E}_{\epsilon \sim \rho_0(\cdot)} \big[ \nabla_\theta f_\theta(s,\epsilon) \nabla_a Q^\pi(s,a)|_{a=f_\theta(s,\epsilon)}\big] + \mathbb{E}_{\epsilon \sim \rho_0(\cdot)} \big[ \gamma p(s\rightarrow s^\prime, 1, f_\theta(s,\epsilon))\nabla_\theta V^\pi(s^\prime)ds^\prime \big]\\
&= \mathbb{E}_{\epsilon \sim \rho_0(\cdot)} \big[ \nabla_\theta f_\theta(s,\epsilon) \nabla_a Q^\pi(s,a)|_{a=f_\theta(s,\epsilon)}\big] +\\
&\ \ \ \ \   \mathbb{E}_{\epsilon \sim \rho_0(\cdot)} \big[ \int_{\mathcal{S}} \gamma p(s\rightarrow s^\prime, 1, f_\theta(s,\epsilon))\mathbb{E}_{\epsilon^\prime \sim \rho_0(\cdot)}[\nabla_\theta f_\theta(s^\prime,\epsilon^\prime)\nabla_a Q^\pi(s^\prime,a^\prime)|_{a^\prime=f_\theta(s^\prime,\epsilon^\prime)}]ds^\prime \big] + \\
&\ \ \ \ \    \mathbb{E}_{\epsilon \sim \rho_0(\cdot)} \big[ \int_{\mathcal{S}} \gamma p(s\rightarrow s^\prime, 1, f_\theta(s,\epsilon^\prime)) \mathbb{E}_{\epsilon^\prime \sim \rho_0(\cdot)} \big[\int_{\mathcal{S}} \gamma p(s^\prime \rightarrow s^{\prime\prime}, 1, f_\theta(s^\prime,\epsilon^\prime))\nabla_\theta V^\pi(s^{\prime\prime})ds^{\prime\prime} \big] ds^\prime \big]\\
&= ...\\
&= \int_{\mathcal{S}} \sum_{t=0}^\infty \gamma^t p(s\rightarrow s^\prime, t, \pi) \mathbb{E}_{\epsilon^\prime \sim \rho_0(\cdot)} \big[\nabla_\theta f_\theta(s^\prime, \epsilon^\prime) \nabla_a Q^\pi(s^\prime,a^\prime)|_{a^\prime = f_\theta(s,\epsilon^\prime)}\big] ds^\prime.
\end{align*}
With the above, we derive the pathwise policy gradient as follows
\begin{align*}
\nabla_\theta J(\pi_\theta) &= \nabla_\theta \int_\mathcal{S} p_1(s) V^\pi(s) ds \\
&= \int_\mathcal{S} p_1(s) \nabla_\theta V^\pi(s) ds\\
&= \int_\mathcal{S} \int_{\mathcal{S}} \sum_{t=0}^\infty \gamma^t p_1(s) p(s\rightarrow s^\prime, t,\pi) ds \mathbb{E}_{\epsilon^\prime \sim \rho_0(\cdot)} \big[\nabla_\theta f_\theta(s^\prime, \epsilon^\prime) \nabla_a Q^\pi(s^\prime,a^\prime)|_{a^\prime = f_\theta(s,\epsilon^\prime)}\big] ds^\prime\\
&= \int_\mathcal{S} \rho_\pi(s^\prime) \mathbb{E}_{\epsilon^\prime \sim \rho_0(\cdot)} \big[\nabla_\theta f_\theta(s^\prime, \epsilon^\prime) \nabla_a Q^\pi(s^\prime,a^\prime)|_{a^\prime = f_\theta(s,\epsilon^\prime)}\big] ds^\prime,
\end{align*}
where $\rho_\pi(s^\prime) = \int_\mathcal{S} \sum_{t=0}^\infty \gamma^t p_1(s) p(s\rightarrow s^\prime, t, \pi)ds$ is the discounted state visitation probability under policy $\pi$. Writing the whole integral as an expectation over states, the policy gradient is
\begin{align*}
\nabla_\theta J(\pi_\theta) &= \mathbb{E}_{s \sim \rho_\pi(s)}\big[\mathbb{E}_{\epsilon \sim \rho_0(\cdot)}[\nabla_\theta f_\theta(s,\epsilon) \nabla_a Q^\pi(s,a)|_{a = f_\theta(s,\epsilon)}]\big].
\end{align*}
which is equivalent to in (\ref{eq:pathwisepolicy}) in \cref{thm:pathgrad}. 
\end{proof}
We can recover the result for deterministic policy gradient by using a degenerate functional form $f_\theta(s,\epsilon) = f_\theta(s)$, i.e. with a deterministic function to compute actions.
\subsection{Unbiased Entropy Gradient}
\begin{restatable}[Optimal Classifier as Density Estimator]{lem}{classifier}
\label{lem:classifier}
Assume $c_\psi$ is expressive enough to represent any classifier (for example $c_\psi$ is a deep neural net). Assume $\mathcal{A}$ to be bounded and let $U(\mathcal{A})$ be uniform distribution over $\mathcal{A}$. Let $\psi^\ast$ be the optimizer to the optimization problem in (\ref{eq:critic}). Then $c_{\psi^\ast}(s,a) = \log \frac{\pi_\theta(a|s)}{|\mathcal{A}|^{-1}}$ and $|\mathcal{A}|$ is the volume of $\mathcal{A}$.
\end{restatable}
\begin{proof}
Observe that (\ref{eq:critic}) is a binary classification problem with data from $a \sim \pi_\theta(\cdot|s)$ against $a\sim \mathcal{U}(\mathcal{A})$. The optimal classifier of the problem produces the density ratio of these two distributions. See for example \cite{goodfellow2015} for a detailed proof.
\end{proof}

\entgrad*
\begin{proof}
Let $\pi_\theta(\cdot|s)$ be the density of implicit policy $a = f_\theta(\epsilon,s), \epsilon\sim \rho_0(\cdot)$. The entropy is computed as follows
\begin{align*}\mathbb{H}\big[\pi_\theta(\cdot|s)\big]= - \mathbb{E}_{a \sim \pi_\theta(\cdot|s)}\big[\log \pi_\theta(a|s)\big].\end{align*}

Computing its gradient 
\begin{align}
\nabla_\theta \mathbb{H}\big[\pi_\theta(\cdot|s)\big] &=  - \nabla_\theta \mathbb{E}_{\epsilon\sim \rho_0(\cdot)}\big[\log \pi_\theta(f_\theta(s,\epsilon)|s)\big] \nonumber \\
&= -\mathbb{E}_{a\sim \pi_\theta(\cdot|s)}\big[\nabla_\theta \log \pi_\theta(a|s)\big]   -\mathbb{E}_{\epsilon\sim \rho_0(\cdot)}\big[\nabla_a \log \pi_\theta(a|s)|_{a=f_\theta(s,\epsilon)} \nabla_\theta f_\theta(s,\epsilon)\big] \nonumber\\
&= -\mathbb{E}_{\epsilon\sim \rho_0(\cdot)}\big[\nabla_a \log \pi_\theta(a|s)|_{a=f_\theta(s,\epsilon)} \nabla_\theta f_\theta(s,\epsilon)\big] \nonumber \\
&= -\mathbb{E}_{\epsilon \sim \rho_0(\cdot)}\big[\nabla_a c_{\psi^\ast}(f(\theta,\epsilon),s) \nabla_\theta f_\theta(s,\epsilon)\big]
\end{align}
In the second line we highlight the fact that the expectation depends on parameter $\theta$ both implicitly through the density $\pi_\theta$ and through the sample $f_\theta(s,\epsilon)$. After decomposing the gradient using chain rule, we find that the first term vanishes, leaving the result shown in the theorem.
\end{proof}

\subsection{Lower Bound}
We recall that given a policy $\pi$, the standard RL objective is $J(\pi) = \mathbb{E}_{\pi}\big[ \sum_{t=0}^\infty \gamma^t r_t \big]$. In maximum entropy formulation, the maximum entropy objective is 
\begin{align}
J_{\text{MaxEnt}}(\pi) = \mathbb{E}_{\pi}\big[\sum_{t=0}^\infty \gamma^t (r_t + \beta \mathbb{H}[\pi(\cdot|s_t)])\big],
\label{eq:maxrlobj}
\end{align}
where $\beta > 0$ is a regularization constant and $\mathbb{H}(\pi(\cdot|s_t))$ is the entropy of policy $\pi$ at $s_t$. We construct a surrogate objective based on another policy $\tilde{\pi}$ as follows
\begin{align}
J_{\text{surr}}(\pi,\tilde{\pi}) = \mathbb{E}_{\pi}\big[\sum_{t=0}^\infty \gamma^t r_t \big] + \beta \mathbb{E}_{\tilde{\pi}}\big[\sum_{t=0}^\infty \gamma^t \mathbb{H}[\pi(\cdot|s_t)]\big].
\label{eq:surrobj}
\end{align}
The following proof highly mimics the proof in \cite{schulman2015}. We have the following definition for coupling two policies
\begin{restatable}[$\alpha-$coupled]{definition}{coupling}
\label{def:coupled}
Two policies $\pi,\tilde{\pi}$ are $\alpha-$coupled if $\mathbb{P}(\pi(\cdot|s) \neq \tilde{\pi}(\cdot|s)) \leq \alpha$ for any $s\in \mathcal{S}$.
\end{restatable}
\begin{restatable}[]{lem}{coupling}
\label{lemma:bounding}
Given $\pi,\tilde{\pi}$ are $\alpha-$coupled, then 
\begin{align*}
|\mathbb{E}_{s_t \sim \pi}\big[\mathbb{H}[\pi(\cdot|s_t)]\big] - \mathbb{E}_{s_t\sim \tilde{\pi}}\big[\mathbb{H}[\tilde{\pi}(\cdot|s_t)]\big]| \leq 2(1-(1-\alpha)^t) \max_{s} |\mathbb{H}[\pi(\cdot|s)]|.
\end{align*}
\end{restatable}
\begin{proof}
 Let $n_t$ denote the number of times that $a_i \neq \tilde{a}_i$ for $i<t$, i.e. the number of times that $\pi,\tilde{\pi}$ disagree before time $t$. We can decompose the expectations as follows
\begin{align*}
\mathbb{E}_{s_t \sim \pi}\big[\mathbb{H}[\pi(\cdot|s_t)]\big] &= \mathbb{E}_{s_t \sim \pi|n_t = 0}\big[\mathbb{H}[\pi(\cdot|s_t)]\big]\mathbb{P}(n_t = 0) + \mathbb{E}_{s_t \sim \pi|n_t > 0}\big[\mathbb{H}[\pi(\cdot|s_t)]\big]\mathbb{P}(n_t > 0),\\
\mathbb{E}_{s_t \sim \tilde{\pi}}\big[\mathbb{H}[\pi(\cdot|s_t)]\big] &= \mathbb{E}_{s_t \sim \tilde{\pi}|n_t = 0}\big[\mathbb{H}[\pi(\cdot|s_t)]\big]\mathbb{P}(n_t = 0) + \mathbb{E}_{s_t \sim \tilde{\pi}|n_t > 0}\big[\mathbb{H}[\pi(\cdot|s_t)]\big]\mathbb{P}(n_t > 0). 
\end{align*}
Note that $n_t = 0$ implies $a_i = \tilde{a}_i$ for all $i<t$ hence
\begin{align*}\mathbb{E}_{s_t \sim \tilde{\pi}|n_t = 0}\big[\mathbb{H}[\pi(\cdot|s_t)]\big] =  \mathbb{E}_{s_t \sim \pi|n_t = 0}\big[\mathbb{H}[\pi(\cdot|s_t)]\big].\end{align*}
The definition of $\alpha-$coupling implies $\mathbb{P}(n_t = 0)\geq (1-\alpha)^t$, and so $\mathbb{P}(n_t > 0)\leq 1 - (1-\alpha)^t$. Now we note that
\begin{align*} |\mathbb{E}_{s_t \sim \pi|n_t > 0}\big[\mathbb{H}[\pi(\cdot|s_t)]\big]\mathbb{P}(n_t > 0) - \mathbb{E}_{s_t \sim \tilde{\pi}|n_t > 0}\big[\mathbb{H}[\pi(\cdot|s_t)]\big]\mathbb{P}(n_t > 0)|\leq 2(1-(1-\alpha)^t)\max_s |\mathbb{H}[\pi(\cdot|s)]|.\end{align*}
Combining previous observations, we have proved the lemma. 
\end{proof}
Note that if we take $\pi = \pi_\theta,\tilde{\pi} = \pi_{\theta_{\text{old}}}$, then the surrogate objective $J_{\text{surr}}(\pi_\theta)$ in (\ref{eq:surrpolicyobj}) is equivalent to $J_\text{surr}(\pi,\tilde{\pi})$ defined in (\ref{eq:surrobj}). With \cref{lemma:bounding}, we prove the following theorem.

\lowerbound*
\begin{proof}
We first show the result for general policies $\pi$ and $\tilde{\pi}$ with $\mathbb{KL}\big[\pi||\tilde{\pi}\big]\leq \alpha$. As a result, \cite{schulman2015} shows that $\pi,\tilde{\pi}$ are $\sqrt{\alpha}-$coupled. Recall the maximum entropy objective $J_{\text{MaxEnt}}(\pi)$ defined in (\ref{eq:maxrlobj}) and surrogate objective in (\ref{eq:surrobj}), take the difference of two objectives
\begin{align*}
|J_{\text{MaxEnt}}(\pi) - J_{\text{surr}}(\pi,\tilde{\pi})| &= \beta|\sum_{t=0}^\infty \gamma^t \mathbb{E}_{s_t \sim \pi}\big[\mathbb{H}[\pi(\cdot|s_t)]\big] - \mathbb{E}_{s_t\sim \tilde{\pi}}\big[\mathbb{H}[\tilde{\pi}(\cdot|s_t)]\big]|\\
&\leq \beta\sum_{t=0}^\infty \gamma^t |\mathbb{E}_{s_t \sim \pi}\big[\mathbb{H}[\pi(\cdot|s_t)]\big] - \mathbb{E}_{s_t\sim \tilde{\pi}}\big[\mathbb{H}[\tilde{\pi}(\cdot|s_t)]\big]|\\
&\leq \beta\sum_{t=0}^\infty \gamma^t (1-(1-\sqrt{\alpha})^t) \max_s |\mathbb{H}[\pi(\cdot|s)]|\\
&= (\frac{1}{1-\gamma} - \frac{1}{1-\gamma(1-\sqrt{\alpha})})\beta\sqrt{\alpha}  \max_s |\mathbb{H}[\pi(\cdot|s)]|\\
&= \frac{\beta\gamma\sqrt{\alpha}}{(1-\gamma)^2}  \max_s |\mathbb{H}[\pi(\cdot|s)]|\\
\end{align*}
Now observe that by taking $\pi = \pi_\theta,\tilde{\pi} = \pi_{\theta_{\text{old}}}$, the above inequality implies the theorem.
\end{proof}
In practice, $\alpha-$coupling enforced by KL divergence is often relaxed \cite{schulman2015,schulman2017}. The theorem implies that, by constraining the KL divergence between consecutive policy iterates, the surrogate objective of entropy regularization maximizes a lower bound of maximum entropy objective.

\section{Operator view of Entropy Regularization and Maximum Entropy RL}
Recall in standard RL formulation, the agent is in state $s$, takes action $a$, receives reward $r$ and transitions to $s^\prime$. Let the discount factor $\gamma < 1$. Assume that the reward $r$ is deterministic and the transitions $s^\prime \sim p(\cdot|s,a)$ are deterministic, i.e. $s^\prime = f(s,a)$, it is straightforward to extend the following to general stochastic transitions. For a given policy $\pi$, define linear Bellman operator as
\begin{align*}\mathcal{T}^\pi Q(s,a) = r + \gamma \mathbb{E}_{a^\prime\sim \pi(\cdot|s^\prime)}\big[Q(s^\prime,a^\prime)\big].\end{align*}

Any policy $\pi$ satisfies the linear Bellman equation $\mathcal{T}^\pi Q^\pi = Q^\pi$. Define Bellman optimality operator (we will call it Bellman operator) as
\begin{align*}\mathcal{T}^\ast Q(s,a) = r + \gamma \max_{a^\prime} Q(s^\prime,a^\prime).\end{align*}
Now we define Mellowmax operator \cite{asadi2017,tuomas2017} with parameter $\beta > 0$ as follows,
\begin{align*}\mathcal{T}_{s} Q(s,a) = r + \gamma \beta \log \int_{a^\prime\in\mathcal{A}} \exp(\frac{Q(s^\prime,a^\prime)}{\beta})da^\prime.\end{align*}
It can be shown that both $\mathcal{T}^\ast$ and $\mathcal{T}$ are contractive operator when $\gamma < 1$. Let $Q^\ast$ be the unique fixed point of $\mathcal{T}^\ast Q = Q$, then $Q^\ast$ is the action value function of the optimal policy $\pi^\ast = \arg\max_\pi J(\pi)$. Let $Q_s^\ast$ be the unique fixed point of $\mathcal{T}_s Q = Q$, then $Q_s^\ast$ is the soft action value function of $\pi_s^\ast = \arg\max_\pi J_{\text{MaxEnt}}(\pi)$. In addition, the optimal policy $\pi^\ast(\cdot|s) = \arg\max_a Q^\ast(s,a)$ and $\pi_s^\ast(a|s) \propto \exp(\frac{Q_s^\ast(s,a)}{\beta})$.

Define Boltzmann operator with parameter $\beta$ as follows
\begin{align*}\mathcal{T}_BQ(s,a) = r + \gamma \mathbb{E}_{a^\prime \sim p_B(\cdot|s^\prime)}\big[Q(s^\prime,a^\prime)\big],\end{align*}

where $p_B(a^\prime|s^\prime) \propto \exp(\frac{Q(s^\prime,a^\prime)}{\beta})$ is the Boltzmann distribution defined by $Q(s^\prime,a^\prime)$. \cite{donoghue2017} shows that the stationary points of entropy regularization procedure are policies of the form $\pi(a|s)\propto \exp(\frac{Q^\pi(s,a)}{\beta})$. We illustrate the connection between such stationary points and fixed points of Boltzmann operator as follows.
\begin{restatable}[Fixed points of Boltzmann Operators]{thm}{boltzmannoperators}
\label{thm:boltzmannoperators}
Any fixed point $Q(s,a)$ of Boltzmann operator $\mathcal{T}_B Q = Q$, defines a stationary point for entropy regularized policy gradient algorithm by $\pi(a|s) \propto \exp(\frac{Q(s,a)}{\beta})$; reversely, any stationary point of entropy regularized policy gradient algorithm $\pi$, has its action value function $Q^\pi(s,a)$ as a fixed point to Boltzmann operator.
\end{restatable}
 \begin{proof}
 Take any fixed point $Q$ of Boltzmann operator, $\mathcal{T}_BQ = Q$, define a policy $\pi(a|s) \propto \exp(\frac{Q(s,a)}{\beta})$. From the definition of Boltzmann operator, we can easily check that $\pi$'s entropy regularized policy gradient is exactly zero, hence it is a stationary point for entropy regularized gradient algorithm. 
 
Take any policy $\pi$ such that its entropy regularized gradient is zero, from \cite{donoghue2017} we know for such policy $\pi(a|s) \propto \exp(Q^\pi(s,a))$. The linear Bellman equation for such a policy $\mathcal{T}^\pi Q^\pi = Q^\pi$ translates directly into the Boltzmann equation $\mathcal{T}_B Q^\pi = Q^\pi$. Hence $Q^\pi$ is indeed a fixed point of Boltzmann operator.
 \end{proof}
The above theorem allows us to associate the policies trained by entropy regularized policy gradient with the Boltzmann operator. Unfortunately, \cite{asadi2017} shows that unlike MellowMax operator, Boltzmann operator does not have unique fixed point and is not a contractive operator in general, though this does not necessarily prevent policy gradient algorithms from converging. We make a final observation that shows that Boltzmann operator $\mathcal{T}_B$ interpolates Bellman operator $\mathcal{T}^\ast$ and MellowMax operator $\mathcal{T}_s$: for any $Q$ and fixed $\beta > 0$ (see \cref{thm:operators}), 
\begin{align}\mathcal{T}^\ast Q \geq \mathcal{T}_B Q \geq \mathcal{T}_s Q.
\label{eq:operators}
\end{align}
If we view all operators as picking out the largest value among $Q(s^\prime,a^\prime),a^\prime \in \mathcal{A}$ in next state $s^\prime$, then $\mathcal{T}^\ast$ is the greediest and $\mathcal{T}_s$ is the most conservative, as it incorporates trajectory entropy as part of the objective. $\mathcal{T}_B$ is between these two operators, since it looks ahead for only one step. The first inequality in (\ref{eq:operators}) is trivial, now we show the second inequality.
\begin{restatable}[Boltzmann Operator is greedier than Mellowmax Operator]{thm}{operators}
\label{thm:operators}
For any $Q \in \mathbb{R}^n$, we have $\mathcal{T}_BQ \geq \mathcal{T}_sQ$, and the equality is tight if and only if $Q_i = Q_j$ for $\forall i,j$.
\end{restatable}
\begin{proof}
Recall the definition of both operators, we essentially need to show the following inequality
\begin{align*}\mathbb{E}_{a^\prime \sim p_B(\cdot|s^\prime)}[Q(s^\prime,a^\prime)] \geq \beta \log \int_{a^\prime \in \mathcal{A}} \exp(\frac{Q(s^\prime,a^\prime)}{\beta}).\end{align*}
Without loss of generality, assume there are $n$ actions in total and let $x_i = \exp(\frac{Q(s^\prime,a_i)}{\beta})$ for the $i$th action. The above inequality reduces to
\begin{align*}\frac{\sum_{i=1}^n x_i \log x_i}{\sum_{i=1}^n x_i} \geq \log \frac{1}{n} \sum_{i=1}^n x_i.\end{align*}
We did not find any reference for the above inequality so we provide a proof below. Notice that $x_i > 0, \forall i$. Introduce the objective $J(x) = \frac{\sum_{i=1}^n x_i \log x_i}{\sum_{i=1}^n x_i} - \log \frac{1}{n} \sum_{i=1}^n x_i$. Compute the gradient of $J(x)$,
\begin{align*}\frac{\partial J(x)}{\partial x_j} = \frac{\sum_{i=1}^n (\log x_j - \log x_i) x_i}{(\sum_{i=1}^n x_i)^2},\ \forall j.\end{align*}
The stationary point at which $\frac{\partial J(x)}{\partial x} = 0$ is of the form $x_i = x_j,\forall i,j$. At such point, let $x_i=x_0,\forall i$ for some generic $x_0$. Then we compute the Hessian of $J(x)$ at such stationary point
\begin{align}
\frac{\partial^2 J(x)}{\partial x_j^2}|_{x_i = x_0,\forall i} &= \frac{n-1}{n^2  x_0^2}, \ \forall j. \nonumber \\
\frac{\partial^2 J(x)}{\partial x_k \partial x_j}|_{x_i = x_0,\forall i} &= \frac{1}{n^2  x_0^2},\ \forall  j\neq k. \nonumber 
\end{align}
Let $H(x_0)$ be the Hessian at this stationary point. Let $t\in \mathbb{R}^n$ be any vector and we can show 
\begin{align*}t^T H(x_0) t = \sum_{i\neq j}\frac{1}{n^2 x_0^2} (t_i - t_j)^2\geq 0,\end{align*}
which implies that $H(x_0)$ is positive semi-definite. It is then implied that at such $x_0$ we will achieve local minimum. Let $x_i = x_0,\forall i$ we find $J(x) = 0$, which implies that $J(x) \geq 0,\forall x$. Hence the proof is concluded.
\end{proof}

\section{Implicit Policy Architecture}
\subsection{Normalizing Flows Policy Architecture}
We design the neural network architectures following the idea of \cite{dinh2015,dinh2017}. Recall that Normalizing Flows \cite{rezende2015} consist of layers of transformations as follows
,\begin{align*}
x = g_{\theta_K} \circ g_{\theta_{K-1}} \circ ... \circ g_{\theta_2} \circ g_{\theta_1} (\epsilon),
\end{align*}
where each $g_{\theta_i}(\cdot)$ is an invertible transformation. We focus on how to design each atomic transformation $g_{\theta_i}(\cdot)$. We overload the notations and let $x,y$ be the input/output of a generic layer $g_{\theta}(\cdot)$, 
\begin{align*}y = g_{\theta}(x).\end{align*}
We design a generic transformation $g_\theta(\cdot)$ as follows. Let $x_{I}$ be the components of $x$ corresponding to subset indices $I \subset \{1,2...m\}$. Then we propose as in \cite{dinh2017},
\begin{align}
y_{1:d} &= x_{1:d} \nonumber \\ 
y_{d+1:m} &= x_{d+1:m} \odot \exp(s(x_{1:d})) + t(x_{1:d}),
\label{eq:couplinglayer}
\end{align}
where $t(\cdot),s(\cdot)$ are two arbitrary functions $t,s : \mathbb{R}^{d} \mapsto \mathbb{R}^{m-d}$. It can be shown that such transformation entails a simple Jacobian matrix
$|\frac{\partial y}{\partial x^T}| = \exp(\sum_{j=1}^{m-d} [s(x_{1:d})]_j)$ where $[s(x_{1:d})]_j$ refers to the $j$th component of $s(x_{1:d})$ for $1\leq j\leq m-d$. For each layer, we can permute the input $x$ before apply the simple transformation (\ref{eq:couplinglayer}) so as to couple different components across layers. Such coupling entails a complex transformation when we stack multiple layers of (\ref{eq:couplinglayer}). To define a policy, we need to incorporate state information. We propose to preprocess the state $s \in \mathbb{R}^n$ by a neural network $L_{\theta_s}(\cdot)$ with parameter $\theta_s$, to get a state vector $L_{\theta_s}(s) \in \mathbb{R}^m$. Then combine the state vector into (\ref{eq:couplinglayer}) as follows,
\begin{align}
z_{1:d} &= x_{1:d} \nonumber \\ 
z_{d+1:m} &= x_{d+1:m} \odot \exp(s(x_{1:d})) + t(x_{1:d}) \nonumber \\
y &= z + L_{\theta_s}(s).
\label{eq:statecouplinglayer}
\end{align}

It is obvious that $x \leftrightarrow y$ is still bijective regardless of the form of $L_{\theta_s}(\cdot)$ and the Jacobian matrix is easy to compute accordingly. 

In our experiments, we implement $s,t$ both as 4-layers neural networks with $k=3$ or $k=6$ units per hidden layer. We stack $K=4$ transformations: we implement (\ref{eq:statecouplinglayer}) to inject state information only after the first transformation, and the rest is conventional coupling as in (\ref{eq:couplinglayer}). $L_{\theta_s}(s)$ is implemented as a feedforward neural network with $2$ hidden layers each with $64$ hidden units. Value function critic is implemented as a feedforward neural network with $2$ hidden layers each with $64$ hidden units with rectified-linear between hidden layers.

\subsection{Non-Invertible Blackbox Policy Architecture}
Any implicit model architecture as in \cite{goodfellow2015,dustin2017} can represent a Non-invertible Blackbox Policy (NBP). On MuJoCo control tasks, consider a task with state space $\mathcal{S} \subset \mathbb{R}^n$ and action space $\mathcal{A} \subset \mathbb{R}^m$. Consider a feedforward neural network with $n$ input units and $m$ output units. The intermediate layers have parameters $\theta$ and the output is a deterministic mapping from the input $a = f_\theta(s)$. We choose an architecture similar to NoisyNet \cite{fortunato2017}: introduce a distribution over $\theta$. In our case, we choose factorized Gaussian $\theta = \mu_\theta + \sigma_\theta\cdot \epsilon$. The implicit policy is generated as 
\begin{align*}a = f_\theta(s),\  \theta = \mu_\theta + \sigma_\theta\cdot\epsilon, \ \epsilon \sim \mathcal{N}(0,1),\end{align*}
which induces an implicit distribution over output $a$. In practice, we find randomizing parameters $\theta$ to generate implicit policy works well and is easy to implement, we leave other approaches for future research. 

In all experiments, we implement the network $f_\theta(\cdot)$ as a feedforward neural network with $2$ hidden layers each with $64$ hidden units. Between layers we use rectified-linear for non-linear activation, layer normalization to standardize inputs, and dropout before the last output. Both value function critic and classifier critic are implemented as feedforward neural networks with $2$ hidden layers each with $64$ hidden units with rectified-linear between hidden layers. Note that $\mu_\theta,\sigma_\theta$ are the actual parameters of the model: we initialize $\mu_\theta$ using standard initialization method and initialize $\sigma_\theta = \log(\exp(\rho_\theta)+1)$ with $\rho_\theta \in [-9.0,-1.0]$. For simplicity, we set all $\rho_\theta$ to be the same and let $\rho = \rho_\theta$. We show below that dropout is an efficient technique to represent multi-modal policy.

\paragraph{Dropout for multi-modal distributions.} Dropout \cite{srivastava2014} is an efficient technique to regularize neural networks in supervised learning. However, in reinforcement learning where overfitting is not a big issue, the application of dropout seems limited. Under the framework of implicit policy, we want to highlight that dropout serves as a natural method to parameterize multi-modal distributions. Consider a feed-forward neural network with output $y \in \mathbb{R}^{m}$.  Assume that the last layer is a fully-connected network with $h$ inputs. Let $x\in\mathbb{R}^n$ be an input to the original neural network and $\phi(x) \in\mathbb{R}^{h}$ be the input to the last layer (we get $\phi(x)$ by computing forward pass of $x$ through the network until the last layer), where $\phi(x)$ can be interpreted as a representation learned by previous layers. Let $W \in \mathbb{R}^{m\times h},b\in\mathbb{R}^m$ be the weight matrix and bias vector of the last layer, then the output is computed as (we ignore the non-linear activation at the output)
\begin{align}
y_i = \sum_{j=1}^{h} \phi_j(x) W_{ij} + b_i , \forall \ 1\leq i\leq m.
\label{eq:fc}
\end{align}
If dropout is applied to the last layer, let $z$ be the Bernoulli mask i.e. $z_i \sim \text{Bernoulli}(p), 1\leq i\leq h$ where $p$ is the probability for dropping an input to the layer. Then 
\begin{align}
y_i = \sum_{j=1}^{h} (\phi_j(x)\cdot z_j) W_{ij} + b_i , \forall \ 1\leq i\leq m
\label{eq:dropout}
\end{align}
Given an $i$, if each $\phi_j(x) W_{ij}$ has a different value, their stochastic sum $\sum_{j=1}^h \phi_j(x) \cdot z_j)W_{ij}$ in (\ref{eq:dropout}) can take up to about $2^h$ values. Despite some redundancy in these $2^h$ values, in general $y_i$ in (\ref{eq:dropout}) has a multi-modal distribution supported on multiple values. We have hence moved from a unimodal distribution (\ref{eq:fc}) to a multi-modal distribution (\ref{eq:dropout}) by adding a simple dropout.

\section{Algorithm Pseudocode}
Below we present the pseudocode for an off-policy algorithm to train NBP. On the other hand, for NFP we can apply any on-policy optimization algorithms \cite{schulman2017,schulman2015} and we omit the pseudocode here.

 \begin{algorithm}[H]
	\begin{algorithmic}[1]
		\STATE INPUT:  target parameter update period $\tau$; learning rate $\alpha_\theta,\alpha_\phi,\alpha_\psi$; entropy regularization constant $\beta$.
		\STATE INITIALIZE: parameters $\theta, \phi, \psi$ and target network parameters $\theta^-,\phi^{-}$; replay buffer $B \leftarrow \{\}$; step counter $counter \leftarrow 0$.
		\FOR {$e=1,2,3...E$}
		\WHILE {episode not terminated}
		\STATE \textbf{// Control}
		\STATE $counter \leftarrow counter + 1$.
		\STATE In state $s$, sample noise $\epsilon \sim \rho_0(\cdot)$, compute action $a =  f_\theta(s,\epsilon)$, transition to $s^\prime$ and receive instant reward $r$.
		\STATE Save experience tuple $\{s,a,r,s^\prime\}$ to buffer $B$.
		\STATE Sample $N$ tuples $D = \{s_j,a_j,r_j,s_j^\prime\}$ from $B$.
		\STATE \textbf{// Update Critic}
		\STATE Compute TD error as in \cite{mnih2013,silver2016} as follows, where $a_j^\prime = f_{\theta^-}(s_j^\prime,\epsilon_j),\epsilon_j \sim \rho_0(\cdot)$. 
\begin{align*}J_\phi = \frac{1}{N}\sum_{j=1}^N (Q_\phi(s_j,a_j) - r_j - \gamma Q_{\phi^-}(s_j^\prime,a_j^\prime))^2.\end{align*}
				\STATE Update $\phi \leftarrow \phi - \alpha_\phi \nabla_\phi J_\phi$.
		\STATE \textbf{// Update classifier}
		\STATE Sample $N$ actions uniformly from action space $a_j^{(u)} \sim \mathcal{U}(\mathcal{A})$. Compute classification objective $C_\psi$ (\ref{eq:critic}) using data $\{s_j,a_j\}_{j=1}^N$ against $\{s_j,a_j^{(u)}\}_{j=1}^N$. 
		\STATE Update classifier $\psi \leftarrow \psi - \alpha_\psi \nabla_\psi C_\psi$.
		\STATE \textbf{// Update policy with entropy regularization}.
		\STATE Compute pathwise gradient $\nabla_\theta J(\pi_\theta)$ (\ref{eq:pathwisepolicy}) with $Q^\pi(s,a)$ replaced by critic $Q_\phi(s,a)$ and states replaced by sampled states $s_j $. 
		\STATE Compute entropy gradient $\nabla_\theta \mathbb{H}\big[\pi_\theta(\cdot|s)\big]$ using (\ref{eq:approxentropy}) on sampled data.
		\STATE Update $\theta \leftarrow \theta + \alpha_\theta (\nabla_\theta J(\pi_\theta) + \beta \nabla_\theta \mathbb{H}\big[\pi_\theta(\cdot|s)\big])$
		\IF {$counter \ \text{mod}\  \tau = 0$}
		\STATE Update target parameter $\phi^{-} \leftarrow \phi, \theta^- \leftarrow\theta $.
		\ENDIF
		\ENDWHILE
		\ENDFOR
	\end{algorithmic}
	\caption{Non-invertible Blackbox Policy (NBF) Off-policy update}
\end{algorithm}

\section{Additional Experiment Results}
\subsection{Locomotion tasks}
As has been shown in previous works \cite{schulman2017}, PPO is almost the most competitive on-policy optimization baseline on locomotion control tasks. We provide a table of comparison among on-policy baselines below. On each task we train for a specified number of time steps and report the average results over $5$ random seeds. Though NFP remains a competitive algorithm, PPO with unimodal Gaussian generally achieves better performance.

\begin{minipage}{\linewidth}
\centering
%\captionof{table}{MuJoCo Benchmark Tasks} \label{tab:title} 
\begin{tabular}{ C{1.2in} C{.55in} *5{C{.45in}}}\toprule[1.5pt]
\bf Tasks & \bf Timesteps  & \bf PPO & \bf A2C & \bf CEM & \bf TRPO & \bf NFP  \\\midrule
 Hopper & $10^6$ & $\mathbf{\approx 2300}$ & $\approx 900$ & $\approx 500$ & $\mathbf{\approx 2000}$ & $\approx 1880$\\ 
 HalfCheetah & $10^6$ & $\mathbf{\approx 1900}$ & $\approx 1000$ & $\approx 500$ & $\approx 0$ & $\mathbf{\approx 2200}$ \\ 
  Walker2d & $10^6$ & $\mathbf{\approx 3500}$ & $\approx 900$ & $\approx 800$ & $\approx 1000$ & $\mathbf{\approx 1980}$\\ 
  InvertedDoublePendulum & $10^6$ & $\mathbf{\approx 8000}$ & $\approx 6500$ & $\approx 0$ & $\approx 0$ & $\mathbf{\approx 8000}$ \\
\bottomrule[1.25pt]
\end {tabular}\par
\bigskip
\small{Table 1: A comparison of NFP with (on-policy) baseline algorithms on MuJoCo benchmark tasks. For each task, we show the average rewards achieved after training the agent for a fixed number of time steps. The results for NFP are averaged over 5 random seeds. The results for A2C, CEM \cite{duanxi2016} and TRPO are approximated based on the figures in \cite{schulman2017}, PPO is from OpenAI baseline implementation \cite{baselines}. We highlight the top two algorithms for each task in bold font. PPO, A2C and TRPO all use unimodal Gaussians. PPO is the most competitive.}
\end{minipage}
\subsection{Multi-modal policy: Fine-tuning for downstream tasks}
In Figure \ref{figure:appendix:trajs}, we compare trajectories generated by agents pre-trained by DDPG and NBP on the running task. Since DDPG uses a deterministic policy, starting from a fixed position, the agent can only run in a single direction. On the other hand, NBP agent manages to run in multiple directions. This comparison partially illustrates that a NBP agent can learn the concept of \emph{general running}, instead of \emph{specialized running} -- running in a particular direction.
\begin{figure}[h]
\centering
\subfigure[Trajectories by DDPG agent]{\includegraphics[width=.46\linewidth]{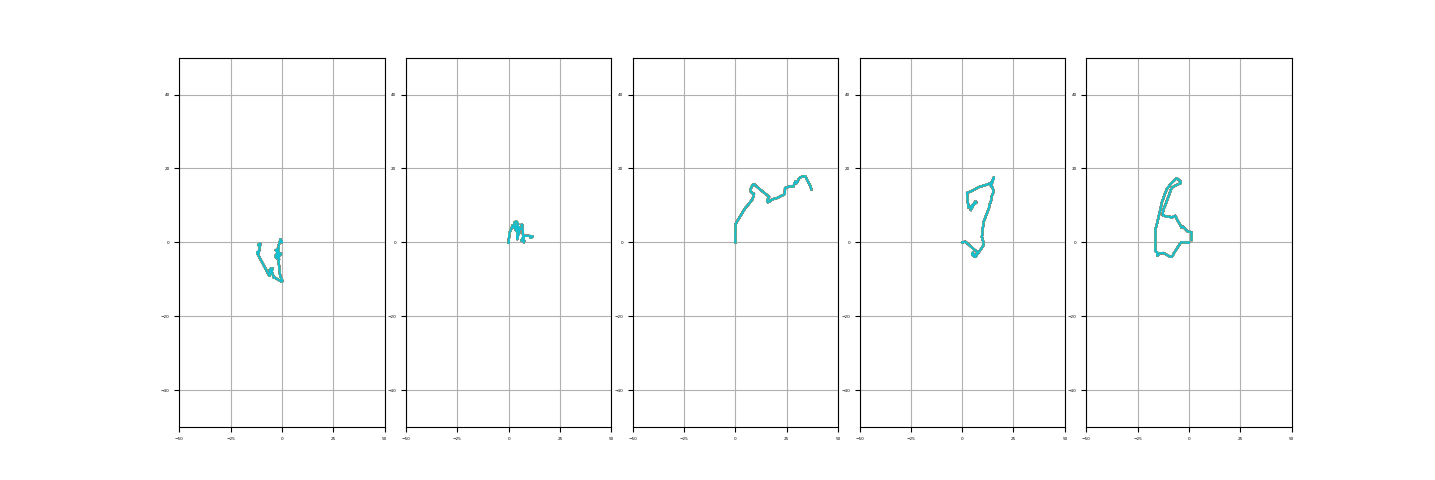}}
\subfigure[Trajectories by NBP agent]{\includegraphics[width=.46\linewidth]{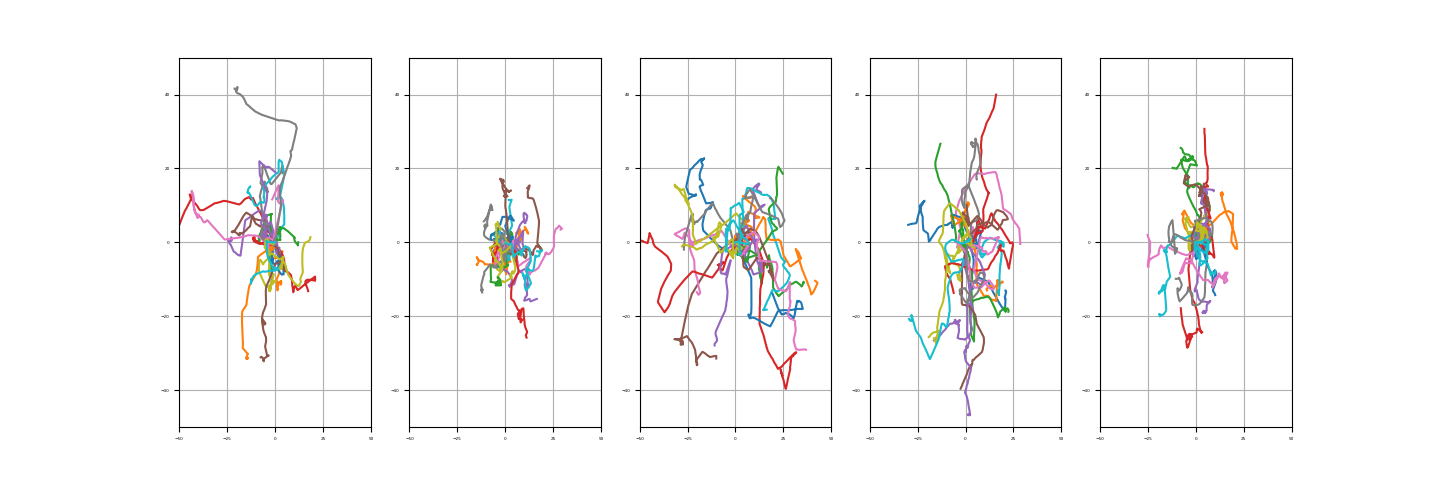}}
%\subfigure[NBP]{\includegraphics[width=.8\linewidth]{ant-actions-multipleseeds}}
\caption{\small{(a)(b): Trajectories generated by DDPG pre-trained agents and NBP pre-trained agent on Ant-Running task under different random seeds. Starting from the initial position, DDPG agent can only produce a single deterministic trajectory due to the deterministic policy; NBP agent produces trajectories that are more diverse, illustrating that NBP agent learns to run in multiple directions.}}
\label{figure:appendix:trajs}
\end{figure}

\subsection{Multi-modal policy: Combining multiple modes by Imitation Learning.}
\paragraph{Didactic Example.} We motivate combining multiple modes of imitation learning with a simple example: imitating an expert with two modes of behavior. Consider a simple MDP on an axis with state space $\mathcal{S} = [-10,10]$, action space $\mathcal{A} = [-1,1]$. The agent chooses which direction to move and transitions according to the equation $s_{t+1} = s_t + a_t$.  We design an expert that commits itself randomly to one of the two endpoints of the state space $s = -10$ or $s = 10$ by a bimodal stochastic policy. We generate $10000$ trajectories from the expert and use them as training data for direct behavior cloning. 

We train a NBP agent using GAN training \cite{goodfellow2015}: given the expert trajectories, train a NBP as a generator that produces similar trajectories and train a separate classifier to distinguish true expert/generated trajectories. Unlike maximum likelihood, GAN training tends to capture modes of the expert trajectories. If we train a unimodal Gaussian policy using GAN training, the agent may commit to a single mode; below we show that trained NBP policy captures both modes.

\begin{figure}[h]
\centering
\subfigure[Actions by Expert]{\includegraphics[width=.32\linewidth]{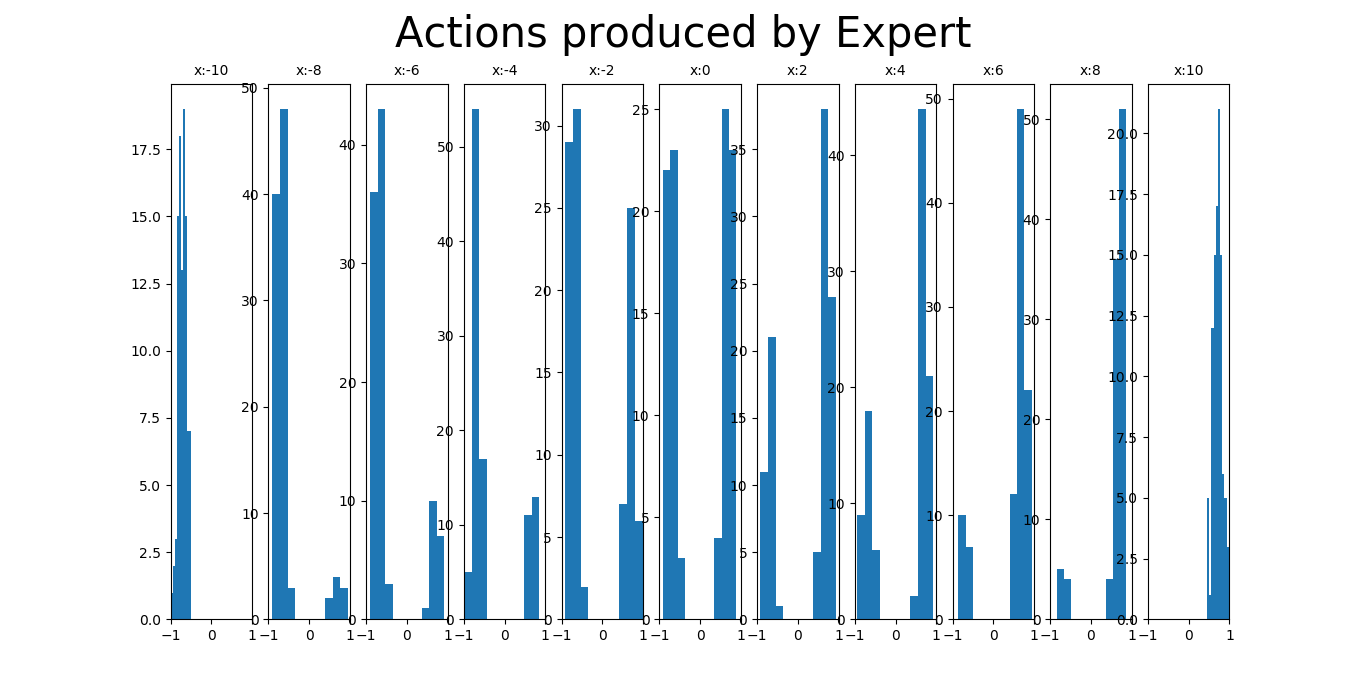}}
\subfigure[Actions by NBP agent]{\includegraphics[width=.32\linewidth]{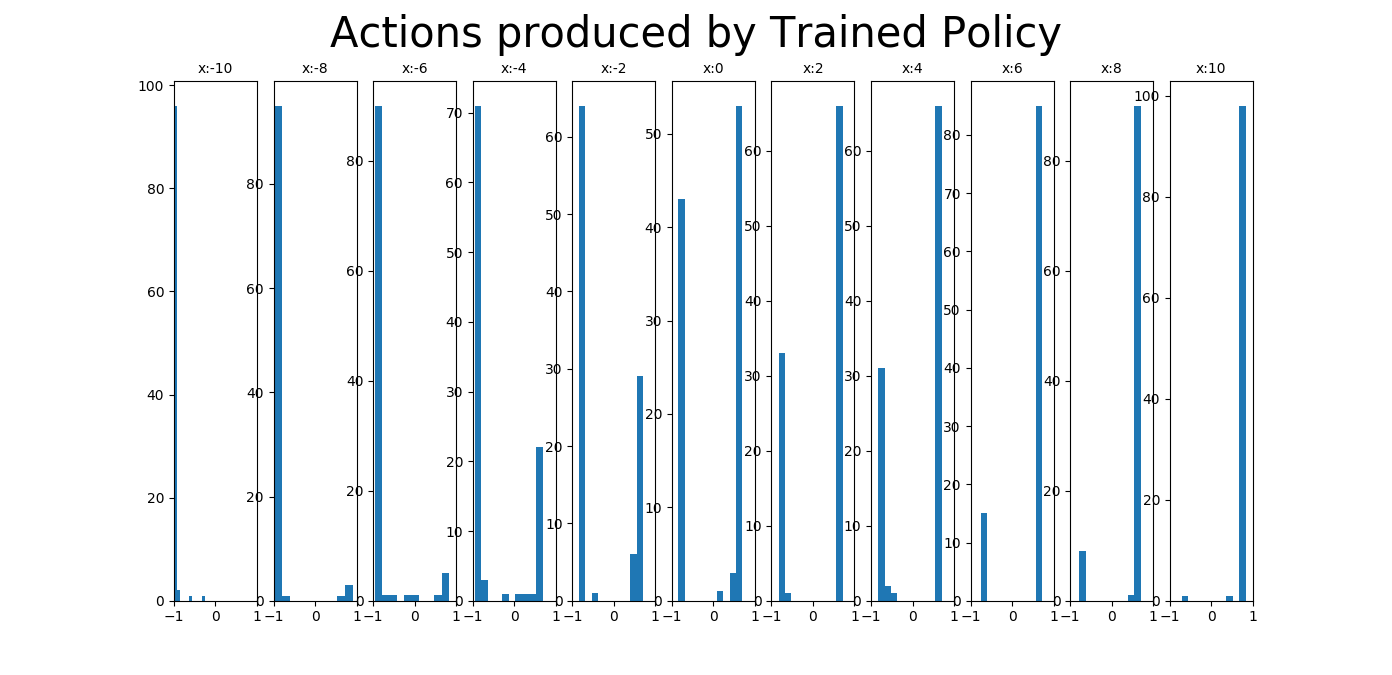}}\\
\subfigure[Trajectories by Expert]{\includegraphics[width=.32\linewidth]{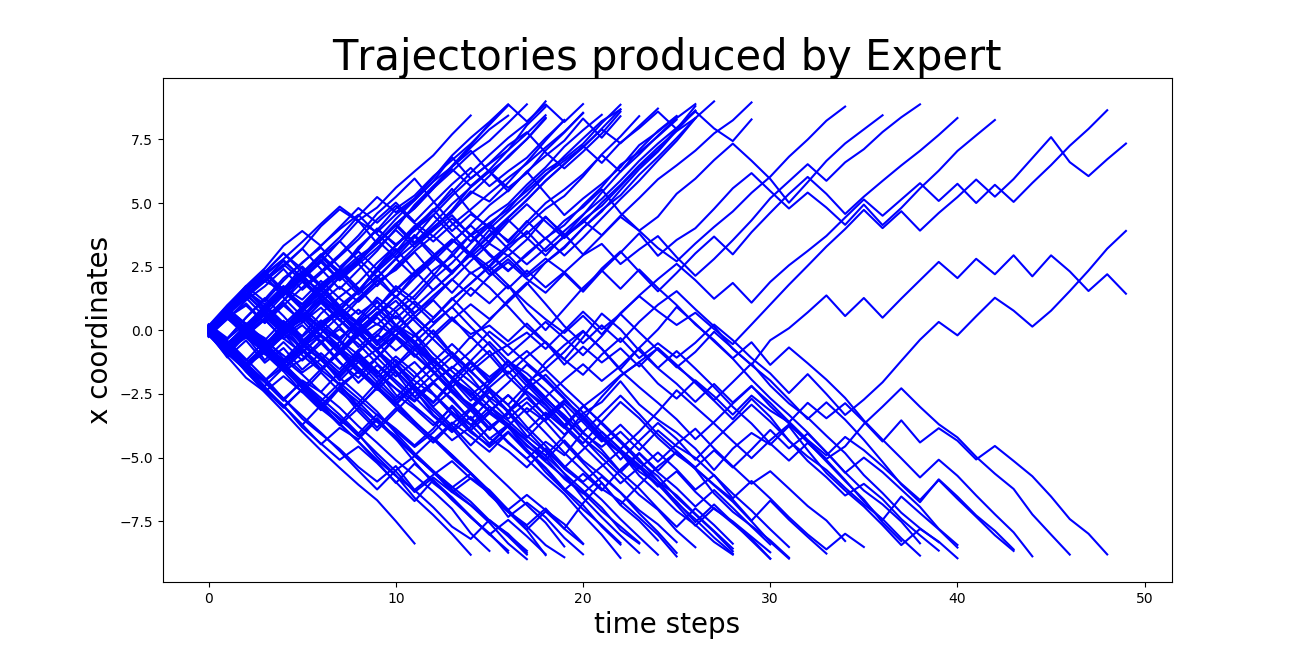}}
\subfigure[Trajectories by NBP agent]{\includegraphics[width=.32\linewidth]{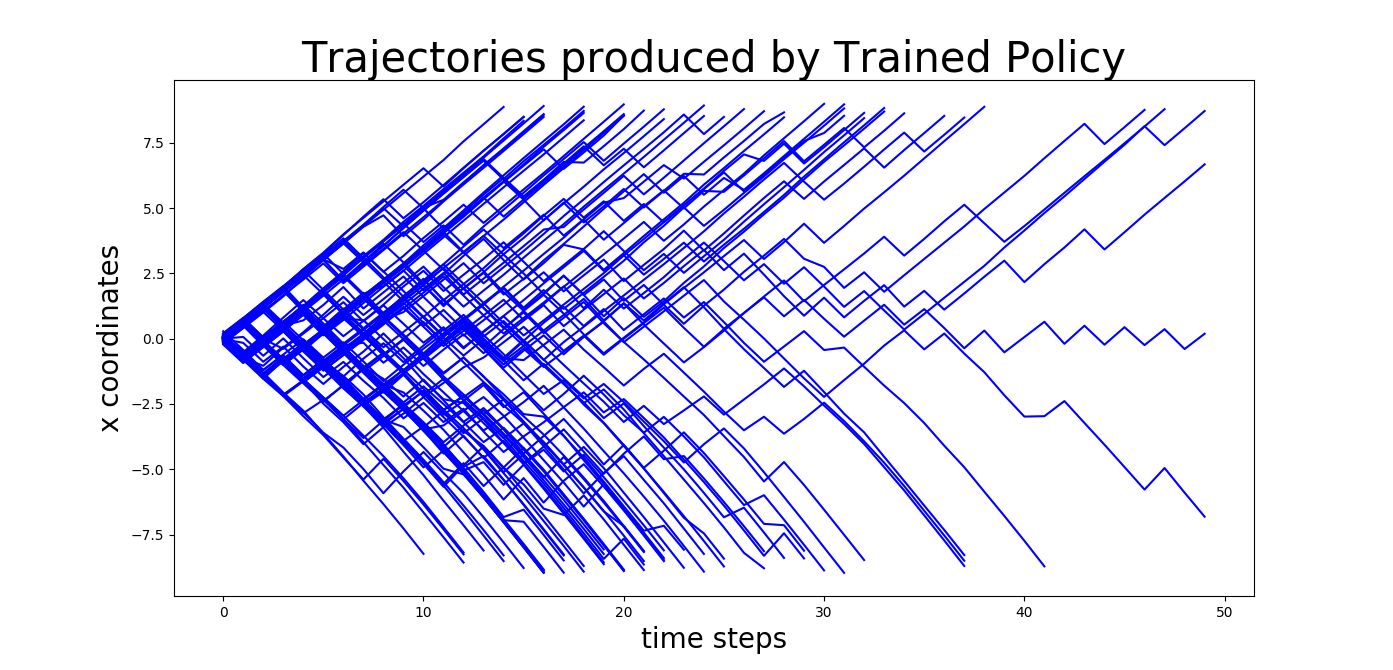}}
\caption{Imitating a bimodal expert: (a)(b) compare the actions produced by the expert and the trained NBP agent at different states $s$. The expert policy has a bimodal policy across different states and becomes increasingly unimodal when $s \approx \pm 10$; the trained policy captures such bimodal behavior. (c)(d) compare the trajectories of expert/trained agent. The vertical axis indicates the states $s$ and horizontal axis is the time steps in an episode, each trajectory is terminated at $s = \pm 10$. Trajectories of both expert and trained policy are very similar.}
\end{figure}
\paragraph{Stochastic Swimmer.} The goal is to train a Swimmer robot that moves either forward or backward. It is not easy to specify a reward function that directly translates into such bimodal behavior and it is not easy to train a bimodal agent under such complex dynamics even if the reward is available. Instead, we train two Swimmers using RL objectives corresponding to two deterministic modes: swimming forward and swimming backward. Combining the trajectories generated by these two modes provides a policy that stochastically commits to either swimming forward or backward. We train a NFP agent (with maximum likelihood behavior cloning \cite{finn2016}) and NBP agent (with GAN training \cite{goodfellow2015}) to imitate expert policy. The trajectories generated by trained policies show that the trained policies have fused these two modes of movement.

\begin{figure}[h]
\centering
\subfigure[Swimmer]{\includegraphics[width=.35\linewidth]{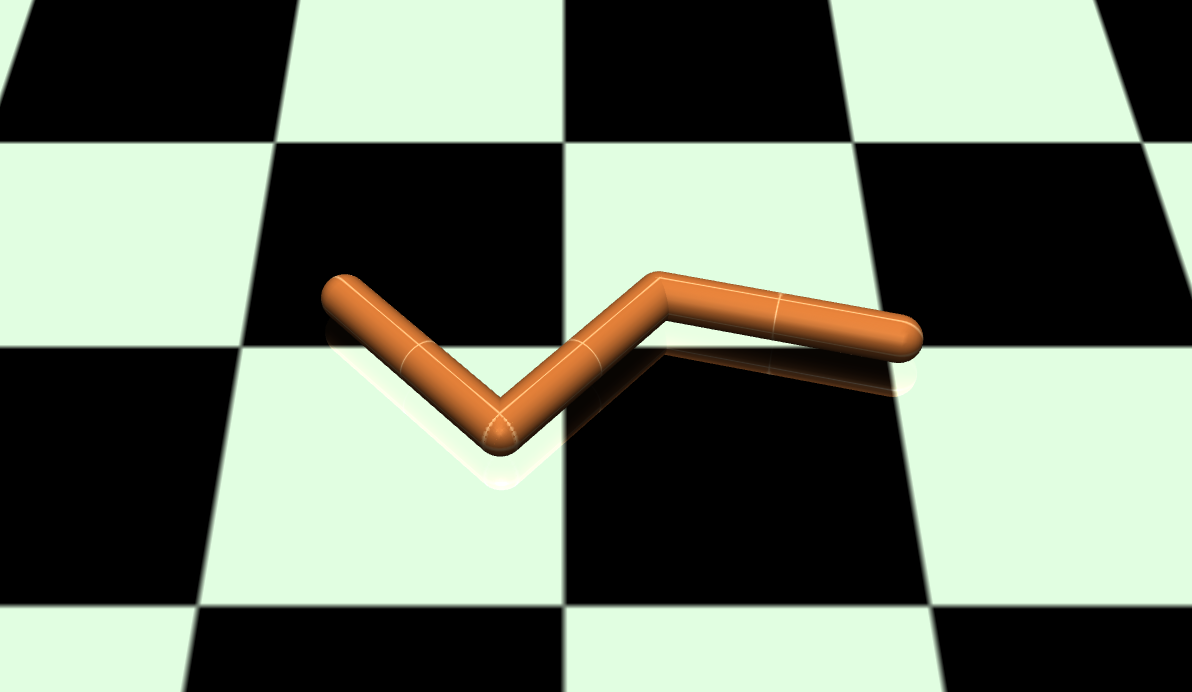}}
\subfigure[Trajectories by Expert]{\includegraphics[width=.5\linewidth]{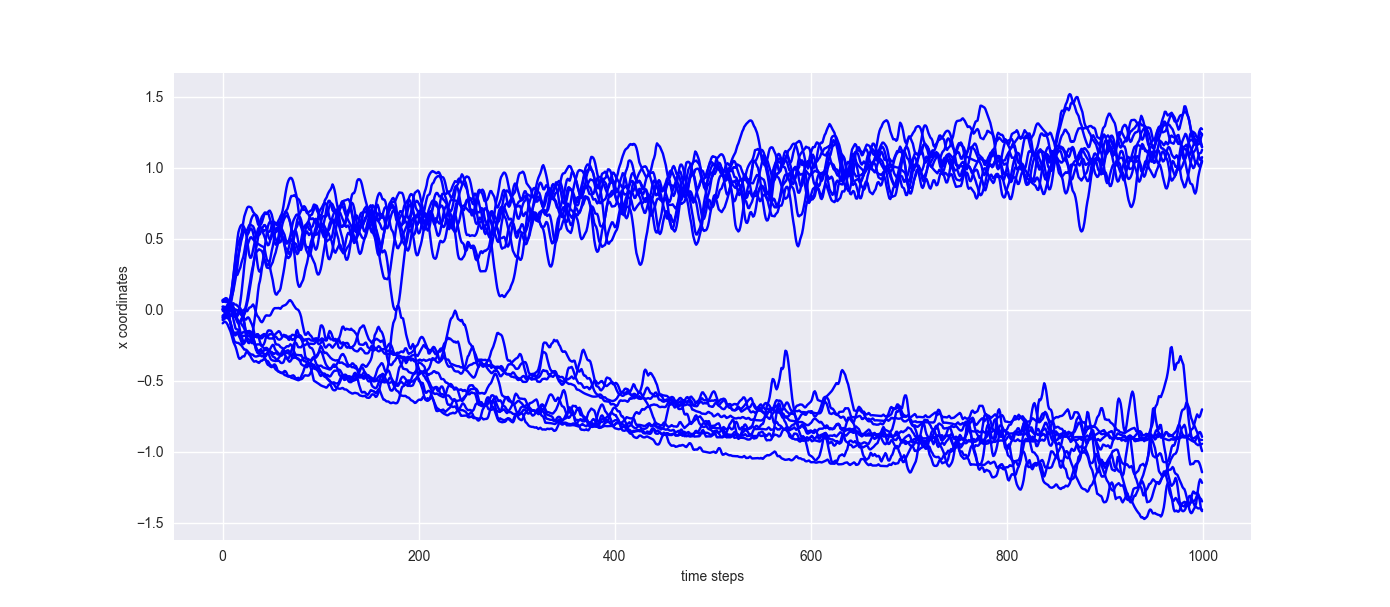}}\\
\subfigure[Trajectories by NBP agent]{\includegraphics[width=.6\linewidth]{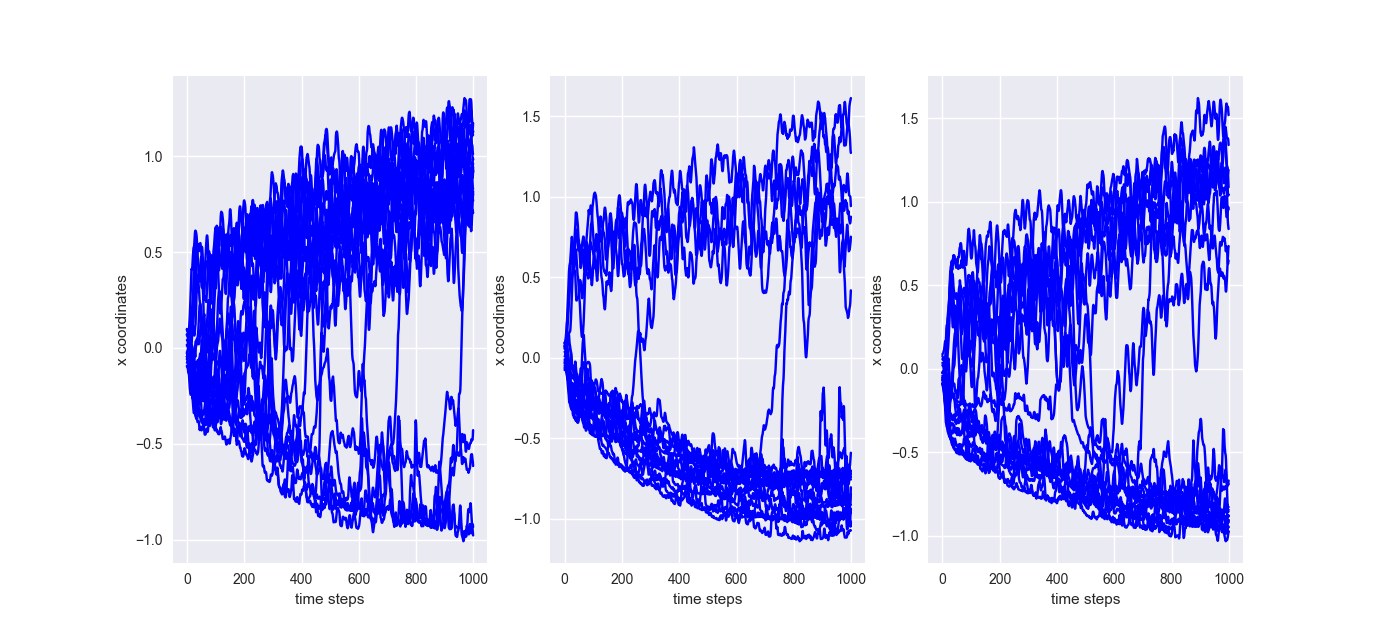}}\\
\subfigure[Trajectories by NFP agent]{\includegraphics[width=.6\linewidth]{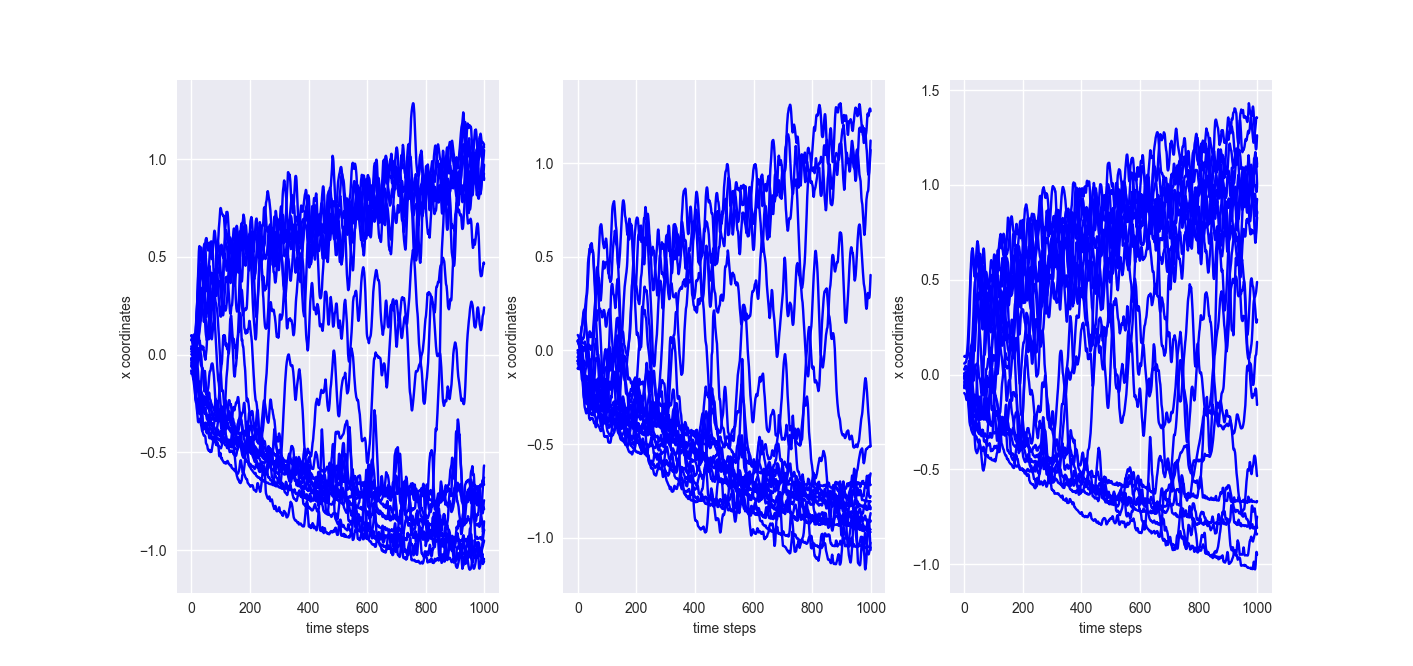}}
\caption{Combining multiple modes by Imitation Learning: stochastic Swimmer. (a) Illustration of Swimmer; (b) Expert trajectories produced by two Swimmers moving in two opposite directions (forward and backward). Vertical axis is the x coordinate of the Swimmer, horizontal axis is the time steps; (c) Trajectories produced by NBP agent trained using GAN under different seeds; (d) Trajectories produced by NFP agent trained using maximum likelihood under different seeds. Implicit policy agents have incorporated two modes of behavior into a single policy, yet a unimodal Gaussian policy can at most commit to one mode.}
\end{figure}

\section{Hyper-parameters and Ablation Study}
\paragraph{Hyper-parameters.} Refer to Appendix C for a detailed description of architectures of NFP and NBP and hyper-parameters used in the experiments. For NFP, critical hyper-parameters are entropy coefficient $\beta$, number of transformation layers $K$ and number of hidden units per layer $k$ for transformation function $s,t$. For NBP, critical hyper-parameters are entropy coefficient $\beta$ and the initialized variance parameter for factorized Gaussian $\rho$. In all conventional locomotion tasks, we set $\beta = 0.0$; for multi-modal policy tasks, we set $\beta \in \{0.1,0.01,0.001\}$. We use Adam \cite{kingma2014adam} for optimization with learning rate $\alpha \in \{3\cdot 10^{-5}, 3\cdot 10^{-4}\}$.

\paragraph{Ablation Study.} For NBP, the default baseline is $\beta = 0.0, \rho = -4.0$. We fix other hyper-parameters and change only one set of hyper-parameter to observe its effect on the performance. Intuitively, large $\rho$ encourages and widespread distribution over parameter $\theta$ and consequently and a more uniform initial distribution over actions. From Figure \ref{fig:ablationnbp} we see that the performance is not monotonic in $\rho,\beta$. We find the model is relatively sensitive to hyper-parameter $\rho$ and a general good choice is $\rho \in \{-4.0,-5.0,-6.0\}$. 

For NFP, the default baseline is $\beta = 0.0, K = 4, k = 3$. We fix other hyper-parameters and change only one set of hyper-parameter to observe its effect on the performance. In general, we find the model's performance is fairly robust to hyper-parameters (see Figure \ref{fig:ablationnfp}): large $K,k$ will increase the complexity of the policy but does not necessarily benefit performance on benchmark tasks; strictly positive entropy coefficient $\beta>0.0$ does not make much difference on benchmark tasks, though for learning multi-modal policies, adding positive entropy regularization is more likely to lead to multi-modality.
\begin{figure}[h]
\centering
\subfigure[Reacher: entropy]{\includegraphics[width=.3\linewidth]{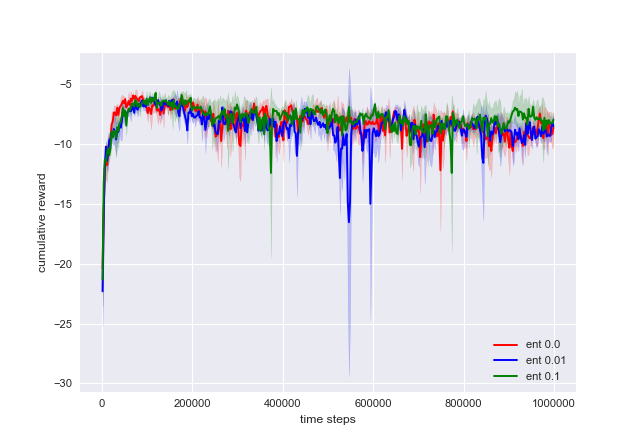}}
\subfigure[Reacher: initial Gaussian variance parameter]{\includegraphics[width=.3\linewidth]{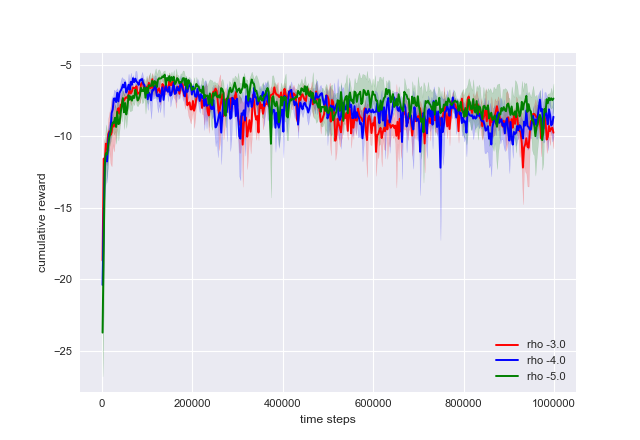}}\\
\subfigure[HalfCheetah: entropy]{\includegraphics[width=.3\linewidth]{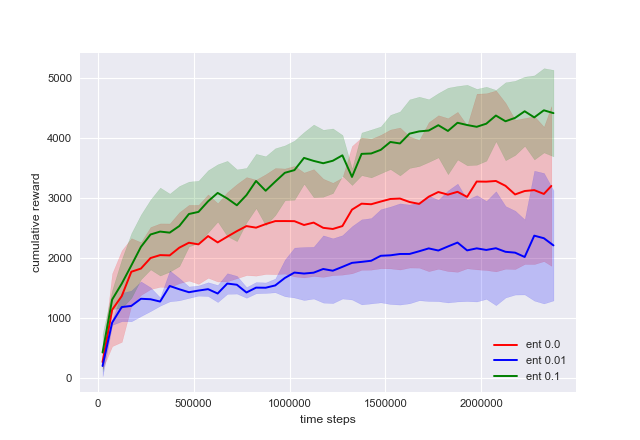}}
\subfigure[HalfCheetah: initial Gaussisn variance parameter]{\includegraphics[width=.3\linewidth]{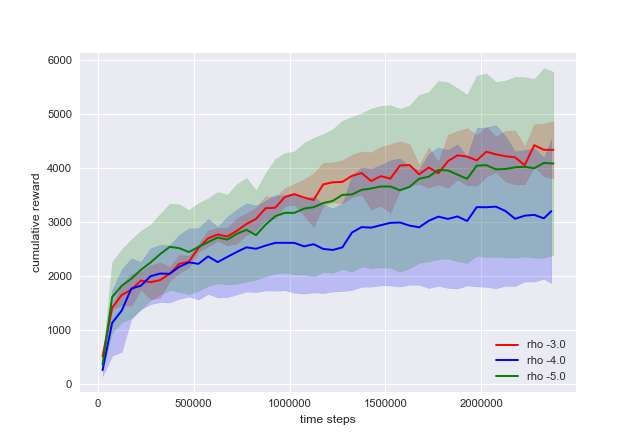}}
\caption{Ablation study: NBP}
\label{fig:ablationnbp}
\end{figure}
\begin{figure}[h]
\centering
\subfigure[Hopper: entropy]{\includegraphics[width=.3\linewidth]{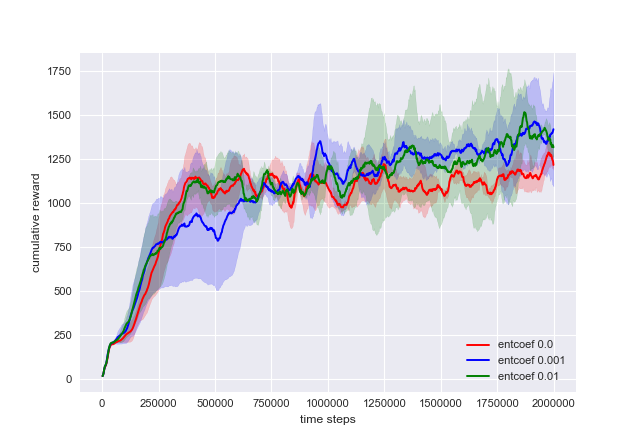}}
\subfigure[Hopper: \# of layers]{\includegraphics[width=.3\linewidth]{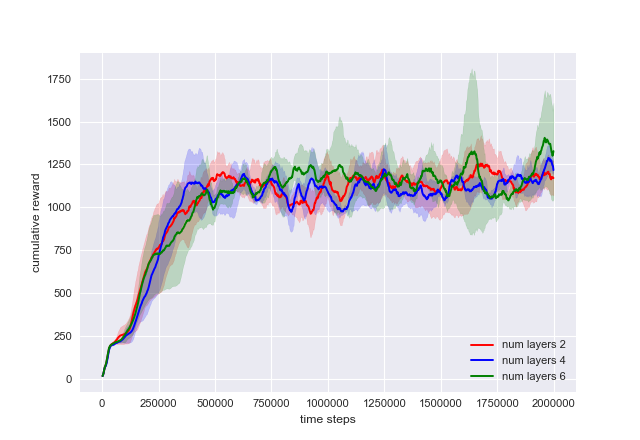}}
\subfigure[Hopper: \# of units]{\includegraphics[width=.3\linewidth]{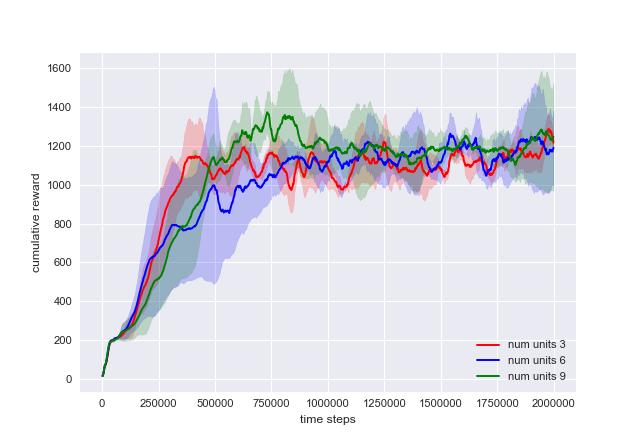}}\\
\subfigure[Reacher: entropy]{\includegraphics[width=.3\linewidth]{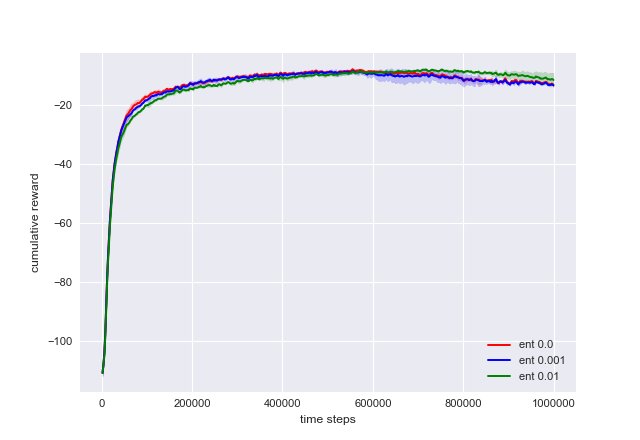}}
\subfigure[Reacher: \# of layers]{\includegraphics[width=.3\linewidth]{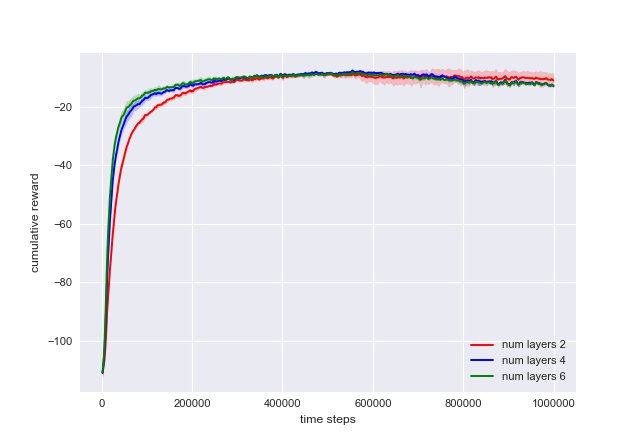}}
\subfigure[Reacher: \# of units]{\includegraphics[width=.3\linewidth]{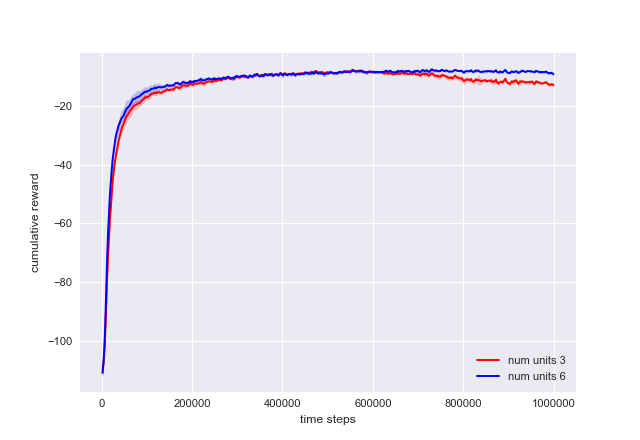}}\\
\subfigure[HalfCheetah: entropy]{\includegraphics[width=.3\linewidth]{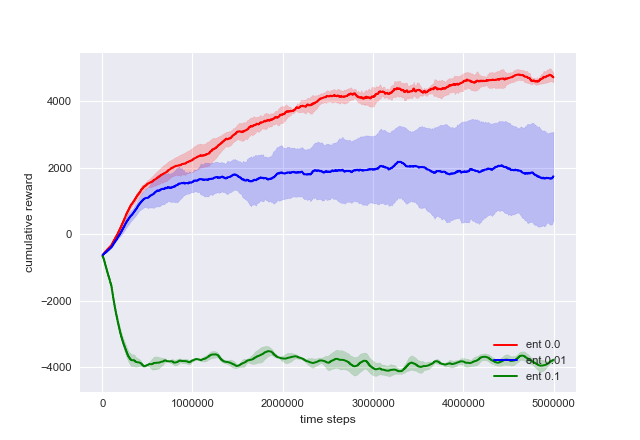}}
\subfigure[HalfCheetah: \# of layers]{\includegraphics[width=.3\linewidth]{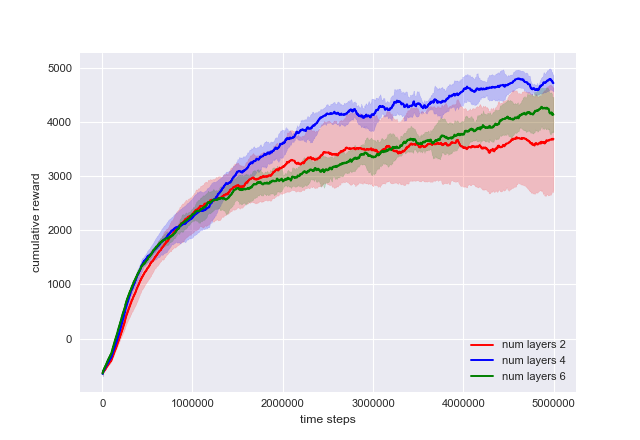}}
\subfigure[HalfCheetah: \# of units]{\includegraphics[width=.3\linewidth]{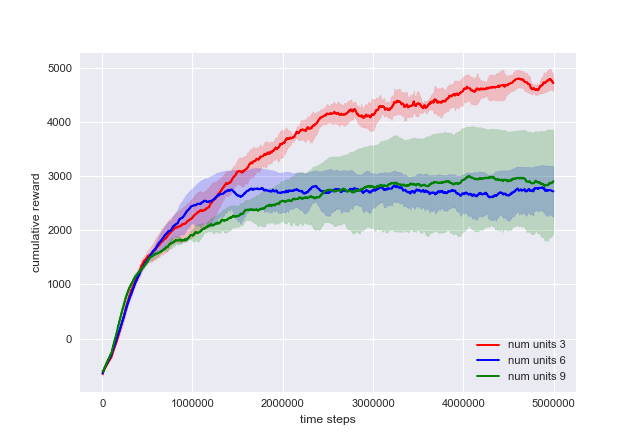}}
\caption{Ablation study: NFP}
\label{fig:ablationnfp}
\end{figure}

\end{document}